\documentclass{article} % For LaTeX2e
\usepackage{misc/iclr2027_conference,times}
\usepackage{amsthm, amsmath, amssymb, bbm, mathabx}
\usepackage{enumitem}
\allowdisplaybreaks % align split page
\usepackage{graphicx}
\usepackage{enumerate}
\usepackage{pifont}
\usepackage{booktabs}
\usepackage{tabularx}
\usepackage{multirow}
\usepackage{multicol}
\usepackage{stfloats}
\usepackage{xspace}
\usepackage{hyperref}
\usepackage{cleveref}
\usepackage{anyfontsize} % customize font sizes
\usepackage{makecell}
\usepackage{hhline}
\usepackage{colortbl}
\usepackage{xpatch}
\usepackage{scalerel,stackengine}
\usepackage{sidecap}
\usepackage[normalem]{ulem}
\usepackage{tcolorbox}
\usepackage{tikz}
\usetikzlibrary{positioning}
\makeatletter
\def\@fnsymbol#1{\ensuremath{\ifcase#1\or *\or \dagger\or \ddagger\or
  \mathsection\or \mathparagraph\or \|\or \diamond \or **\or \dagger\dagger
  \or \ddagger\ddagger \else\@ctrerr\fi}}
\makeatother

\makeatletter
\newcommand{\printfnsymbol}[1]{%
  \textsuperscript{\@fnsymbol{#1}}%
}
\makeatother

\newtheorem{theorem}{Theorem}
\newtheorem{lemma}{Lemma}
\newtheorem{corollary}[theorem]{Corollary}

\crefname{claim}{claim}{claims}
\Crefname{claim}{Claim}{Claims}
\newtheorem{proposition}[theorem]{Proposition}

\crefname{condition}{Condition}{Conditions}

\newtheorem{assumption}{Assumption}
\crefname{assumption}{Assumption}{Assumptions}
\Crefname{assumption}{Assumption}{Assumptions}

\theoremstyle{definition}
\newtheorem{remark}{Remark}

\DeclareMathOperator*{\argmax}{arg\,max}

\newcommand{\Roma}[1]{\uppercase\expandafter{\romannumeral#1}}

\newcommand{\TopK}{\operatorname{TopK}}
\newcommand{\adam}{{\rm adam}}

\newcommand{\lm}[1]{\textbf{\textsc{#1}}}
\newcommand{\rebut}[1]{{#1}}

\usepackage{amsmath,amsfonts,bm}

\def\eqref#1{equation~\ref{#1}}
\def\1{\bm{1}}

\DeclareMathAlphabet{\mathsfit}{\encodingdefault}{\sfdefault}{m}{sl}
\SetMathAlphabet{\mathsfit}{bold}{\encodingdefault}{\sfdefault}{bx}{n}

\usepackage{url}
\usepackage{graphicx}
\usepackage{float}

\definecolor{darkblue}{rgb}{0, 0, 0.8}
\hypersetup{colorlinks=true, allcolors=darkblue}

\title{Ranking and Partitioning of Attention Heads for Sparse and Continual Reinforcement Learning Fine-tuning}
\author{\textbf{Ruizhe Shi}$^{1}$\quad \textbf{Zaiwen Yang}$^{2}$\quad \textbf{Qiubo Wang}$^{3}$\quad \textbf{Banghua Zhu}$^{1}$\quad \textbf{Maryam Fazel}$^{1,4}$\vspace{3pt}\\
$^1$University of Washington\quad
$^2$Tsinghua University\quad 
$^3$University of Hong Kong\quad $^4$Amazon, Inc.\\
}
\iclrfinalcopy % Uncomment for camera-ready version, but NOT for submission.

\begin{document}

\maketitle

\begin{abstract}
We study sparse and continual reinforcement learning fine-tuning of language models through attention-head ranking and partitioning. For sparse learning, we rank attention heads by gradient norm to identify a small subset of heads that can be trained effectively. For continual learning, we partition heads across tasks to reduce cross-task interference and preserve previously learned capabilities. Theoretically, we relate gradient norm to training effectiveness, and further characterize why sparse, separated head updates can remain stable. Empirically, across math, instruction-following, and coding tasks, tuning higher-ranked heads consistently outperforms tuning lower-ranked ones while often matching denser baselines, and head partitioning improves retention while preserving adaptation to new tasks.
\end{abstract}
\section{Introduction}
Reinforcement-learning (RL) fine-tuning has become an important approach for improving the capabilities of language models~(LMs). However, RL fine-tuning is expensive, especially when a large fraction of model parameters is updated, which increases gradient memory and synchronization costs in distributed training~\citep{Miahi2026UnderstandingAE,fireworks2026frontier_rl}. Recent works suggest that effective adaptation can often be achieved by updating only restricted parts of the LM, such as attention modules or individual transformer layers~\citep{jin2025finetuning,Zhang2026IsOL}. This motivates sparse RL fine-tuning under a limited trainable-parameter budget.

In this work, we focus on attention heads as the basic units for sparse RL fine-tuning. Classical pruning studies show that many attention heads can be removed with little performance degradation, while a smaller subset plays more important roles~\citep{michel2019,voita2019analyzing}. More recent studies further suggest that attention heads can specialize toward different latent signals during training~\citep{sagitova2026}, and that high-level LM capabilities can be localized to sparse, task-specific subsets of heads~\citep{bair2026compressed}. This motivates us to design a criterion to identify those heads that are most effective to update during training.

Sparse RL fine-tuning is also appealing for continual learning, where sequential updates can interfere with previously learned capabilities and lead to catastrophic forgetting. A common approach is to constrain different tasks to parameter or gradient subspaces~\citep{chaudhry2020continual,doan2021theoretical,wang2023orthogonal,liang2024inflora,qiu2026splitlora}. Sparse updates can therefore serve to separate task-specific adaptations. Such parameter isolation avoids direct overwriting, yet indirect functional interference may remain~\citep{Lanzillotta2023TowardsGF}, requiring careful selection of the trainable parameters such as \citet{iurada2025efficient,Lin2025ContinualLV,lin2026parameter}.

\textbf{Contributions.} We study attention-head ranking and partitioning for sparse and continual RL fine-tuning. Our main contributions are summarized as follows:
\begin{itemize}[leftmargin=*, itemsep=1pt, topsep=2pt]
\item We propose gradient norm as a simple criterion for ranking attention heads. Across several tasks, tuning higher-ranked heads consistently outperforms tuning lower-ranked heads, while often matching or exceeding denser baselines.
\item We propose partitioning attention heads across sequential tasks and study several head-allocation strategies. Our experiments show that head partitioning with these strategies effectively prevents forgetting while preserving adaptation to new tasks.
\item We provide a theoretical analysis to guide both designs. We characterize the role of gradient norm in the optimization dynamics of sparse and continual learning, and further analyze the second-order structure of multi-head self-attention, revealing favorable properties of attention-head blocks for parameter and functional isolation.
\end{itemize}
\vspace{-0.1in}
\section{Related Works}
\vspace{-0.1in}
\textbf{Sparse optimization in deep learning and LMs.} Sparse optimization has been studied through subset selection~\citep{Dan2019OptimalAO}, sparse reward-model structure~\citep{yao2025leveraging,shi2026understanding}, and sparse task-specific parameter updates~\citep{guo-etal-2021-parameter}. For LMs, recent works show that updating only attention modules~\citep{jin2025finetuning} or a single layer~\citep{Zhang2026IsOL} can obtain competitive performance, while sparse updates can improve communication efficiency in distributed RL~\citep{Miahi2026UnderstandingAE,fireworks2026frontier_rl}. 

\textbf{Subspace continual learning.} A common approach to continual learning is to reduce interference by constraining new-task updates to subspaces that avoid directions important for previous tasks. Early methods construct such subspaces from past gradients or representations~\citep{farajtabar2020orthogonal,saha2021gradient,chaudhry2020continual,doan2021theoretical}. Recent works extend this idea to LoRA adaptation~\citep{wang2023orthogonal,liang2024inflora,zhang2025lori,qiu2026splitlora}, full fine-tuning~\citep{nayak2026sculpting}, and model merging~\citep{qiao2026merge,tang2025merging,yang2025continual}.
\citet{iurada2025efficient} rank parameters by sensitivity measured using Fisher matrix, and fine-tune only the least-sensitive ones. \citet{Lin2025ContinualLV} update only parameters that are activated more strongly by new data than by pretraining data, reducing interference with prior knowledge. \citet{lin2026parameter} dynamically update parameter-level masks based on accumulated gradient magnitudes, whereas we use static head-level gradient norms. \rebut{\citet{panos2025efficient} use a similar gradient-based head-selection criterion. We further investigate how to allocate attention heads across complex and heterogeneous RL tasks to mitigate forgetting, whereas \citet{panos2025efficient} mainly considers continual learning across different object categories in visual adaptation.}

\textbf{Theoretical understanding of continual learning.} Neural tangent kernel~(NTK)-based analyses relate catastrophic forgetting to overlap between task-induced tangent subspaces~\citep{doan2021theoretical,zheng2026understanding}, while kernel-based views have also been developed for LM fine-tuning~\citep{Malladi2023}. Parameter isolation has been studied through Hessian-based arguments~\citep{Lanzillotta2023TowardsGF}. In RL, \citet{luo2026sphere} characterize loss of plasticity through degradation of spectral plasticity, while \citet{shenfeld2026rls} show that online RL tends to forget less by staying closer to the base policy, and \citet{chen2026retaining} attribute this advantage to the use of on-policy data. We connect the stability--plasticity trade-off in continual learning directly to gradient norms.

\textbf{Mechanism of attention heads.} Prior work has shown that attention heads are highly heterogeneous, with many heads being redundant while a smaller subset exhibits specialized behavior~\citep{michel2019,voita2019analyzing}. More directly related to this work, \citet{bair2026compressed} show that small subsets of attention heads can be critical to performance on particular LM capabilities. \citet{sagitova2026} theoretically characterize how attention heads specialize toward distinct latent signals during training. \citet{ormaniec2024} further analyzes the second-order structure of self-attention. Together, these results motivate us to focus on attention heads as structured units.
\vspace{-0.1in}
\section{Preliminaries}
\vspace{-0.1in}
\label{sec:preliminaries}
\textbf{Notations.}
Let $[H]:=\{1,\ldots,H\}$ index the attention heads. We partition the trainable attention parameters into disjoint head blocks $\theta=(\theta_1,\ldots,\theta_H)$, where $\theta_h$ denotes the parameters associated with head $h$. For $S\subseteq[H]$, let $\theta_S:=(\theta_h)_{h\in S}$ and $S^c:=[H]\setminus S$. More generally, we write $x_h$ and $x_S$ for the corresponding vector blocks, and $D_{hh'}$ and $D_{ST}$ for the corresponding matrix blocks, for $h,h'\in [H],S,T\subseteq[H]$.
For an update $\Delta=(\Delta_1,\ldots,\Delta_H)$, define its head support by $\operatorname{supp}_{\mathrm{head}}(\Delta):=\{h\in[H]:\|\Delta_h\|_2>0\}$. Let $\{\xi_i\}_{i=1}^N$ denote the training data and $\mathcal L(\theta):=\frac{1}{N}\sum_{i=1}^N\mathcal L(\theta;\xi_i)$ the empirical training loss. In the RL fine-tuning setting, each $\xi_i$ is a prompt, and $\mathcal L(\theta;\xi_i)$ denotes the corresponding empirical RL loss induced by rollouts from that prompt. We use $\|\cdot\|_2$ for the Euclidean norm and $\|A\|_{\mathrm{op}}$ for the operator norm of a matrix $A$.

\textbf{Gradient norm.}
Given a model with parameters $\theta$, we define the gradient norm of a head set $S$ as
\begin{align*}
\mathcal G(S;\theta):=\|\nabla_{\theta_S}\mathcal L(\theta)\|_2
=\sqrt{\sum_{h\in S}\|\nabla_{\theta_h}\mathcal L(\theta)\|_2^2}\;.
\end{align*}
For two tasks $A$ and $B$, we write $\mathcal G_A(S;\theta):=\|\nabla_{\theta_S}\mathcal L_A(\theta)\|_2$ and $
\mathcal G_B(S;\theta):=\|\nabla_{\theta_S}\mathcal L_B(\theta)\|_2$.
We assume that these losses are twice continuously differentiable on the relevant training regions.

\textbf{Second-order terms of loss Hessian.}
We study the directional second-order derivative:
\begin{align*}
    \frac{\left|\nabla_\theta\mathcal L(\theta)^\top \nabla_\theta^2\mathcal L_(\theta)\nabla_\theta\mathcal L(\theta)\right|}{\|\nabla_\theta\mathcal L(\theta)\|_2^2}\;,
\end{align*}
which measures how rapidly the gradient norm changes when moving along its own normalized gradient direction. For two objectives $\mathcal L_A$ and $\mathcal L_B$, we also study the mixed directional second-order derivative:
\begin{align*}
    \frac{\left|\nabla_\theta\mathcal L_A(\theta)^\top \nabla_\theta^2\mathcal L_A(\theta)\nabla_\theta\mathcal L_B(\theta)\right|}{\|\nabla_\theta\mathcal L_A(\theta)\|_2\|\nabla_\theta\mathcal L_B(\theta)\|_2}\;,
\end{align*}
which measures how rapidly the task-$A$ gradient norm changes when moving along the normalized task-$B$ gradient direction.

\textbf{Multi-head self-attention.}
Consider a multi-head self-attention layer with input $X\in\mathbb R^{L\times d}$ and $H$ heads. For head $h\in[H]$, let $W_h^Q,W_h^K\in\mathbb R^{d\times d_k}$ and $W_h^V\in\mathbb R^{d\times d}$, and define
\begin{align*}
Q_h=XW_h^Q,\quad K_h=XW_h^K,\quad V_h=XW_h^V,\quad A_h=\operatorname{softmax}\left(\frac{Q_hK_h^\top}{\sqrt{d_k}}\right)\;.
\end{align*}
The layer output is $F(\theta)=\sum_{h=1}^H A_hV_h\in\mathbb R^{L\times d}$.
\section{Ranking the Attention Heads}\label{sec:ranking}
\subsection{Theoretical %understandings
\rebut{explanation}}\label{sec:ranking-theory}
We now study gradient-norm ranking for attention heads. Given the base model $\theta_0$, let $\mathcal G(S):=\mathcal G(S;\theta_0)$ denote the gradient norm defined in \Cref{sec:preliminaries}, and let $S_k$ be the set of the $k$ heads with largest individual gradient norms at $\theta_0$. Omitted proofs are deferred to Appendix~\ref{sec:omitted-ranking-theory}. We first characterize what this quantity measures under gradient flow. 
\begin{proposition}
\label{prop:gradnorm_flow_eq}
Let $\theta(t;S)$ follow Euclidean gradient flow restricted to $S$, initialized at $\theta_0$, and suppose the flow exists on $[0,T]$. Then $-\frac{d}{dt}\mathcal L(\theta(t;S))=\mathcal G(S;\theta(t))^2$.
\end{proposition}
\begin{proof}
The gradient flow satisfies $\dot\theta_S(t;S)=-\nabla_{\theta_S}\mathcal L(\theta(t;S))$ and $\dot\theta_{S^c}(t;S)=0$, so the chain rule gives $\frac{d}{dt}\mathcal L(\theta(t;S))=-\|\nabla_{\theta_S}\mathcal L(\theta(t;S))\|_2^2=-\mathcal G(S;\theta(t))^2$.
\end{proof}

\begin{corollary}
\label{cor:gradnorm-initial}
We have $S_k\in\argmax_{|S|\le k}\left\{-\left.\frac{d}{dt}\mathcal L(\theta(t;S))\right|_{t=0}\right\}$, which maximizes the initial loss decrease rate.
\end{corollary}

\begin{corollary}
\label{prop:gradnorm-ranking}
For any head set $S$ and any $t>0$ for which the restricted flow exists, define the time-averaged gradient norm by $\overline{\mathcal G}(S;t):=\sqrt{\frac{1}{t}\int_0^t\mathcal G(S;\theta(s;S))^2\,ds}$.
Then we have
\begin{align*}
    \mathcal L(\theta_0)-\mathcal L(\theta(t;S_k))\rebut{=}\frac{\overline{\mathcal G}(S_k;t)^2}{\max_{|S|\le k}\overline{\mathcal G}(S;t)^2}\max_{|S|\le k}\left[\mathcal L(\theta_0)-\mathcal L(\theta(t;S))\right]\;.
\end{align*}
\end{corollary}
\begin{remark}
This simple gradient-flow analysis shows that gradient norm characterizes training efficiency under a fixed head budget. It determines both the initial loss-decrease rate and finite-time progress through the time-averaged gradient norm.
\end{remark}
\begin{assumption}\label{assum:sparse_smooth}
    \rebut{Suppose that, for some $\beta>0$,
    \begin{align*}
        \frac{\left|(\nabla_{\theta_S}\mathcal L(\theta(t;S)))^\top\nabla_{\theta_S}^2\mathcal L(\theta(t;S))\nabla_{\theta_S}\mathcal L(\theta(t;S))\right|}{\Vert\nabla_{\theta_S}\mathcal L(\theta(t;S))\Vert_2^2}\le\beta
    \end{align*}
for all $t\in[0,T]$ with $\Vert\nabla_{\theta_S}\mathcal L(\theta(t;S))\Vert_2^2>0$.}
\end{assumption}
\begin{proposition}
\label{prop:gradnorm_flow}
\rebut{Under \Cref{assum:sparse_smooth},} for every $t\in(0,T]$,
\begin{align*}
    \sqrt{\frac{1-e^{-2\beta t}}{2\beta t}}\,\mathcal G(S)\le\overline{\mathcal G}(S;t)\le\sqrt{\frac{e^{2\beta t}-1}{2\beta t}}\,\mathcal G(S)\;.
\end{align*}
\end{proposition}
\begin{remark}
Under \rebut{\Cref{assum:sparse_smooth}}, the time-averaged gradient norm remains controlled by its initial value, so the ranking need not be recomputed during training as in \citet{lin2026parameter}. \rebut{The typical values of $\beta$ are reported in \Cref{tab:sparse-second-order} in Appendix~\ref{sec:other_supp_tables}. Note that, $N$ gradient descent steps with learning rate $\eta$ correspond to $T=N\eta$. In our experiments, the learning rate is typically $3\times10^{-6}$. We thus have $\sqrt{\frac{e^{2\beta t}-1}{2\beta t}}\le 3$ and $\sqrt{\frac{1-e^{-2\beta t}}{2\beta t}}\ge 0.5$ in most cases. The initial ranking empirically remains instructive for hundreds of training steps, as shown in \Cref{tab:ranking-consistency} in Appendix~\ref{sec:other_supp_tables}.} See \Cref{sec:attention-curvature} for discussions on why this second-order assumption on attention blocks is reasonable.
\end{remark}
\vspace{-0.1in}
\textbf{Adam's perspective.}
We now give a complementary optimizer-level interpretation based on Adam. Let $m$ and $v$ denote the moment estimates used in the Adam update at $\theta_0$, and let
\begin{align*}
    D:=\operatorname{diag}(\sqrt v+\epsilon_{\adam})\succ0,
    \qquad \eta>0\;,
\end{align*}
where $\epsilon_{\adam}>0$. The Adam step $\Delta^{\adam}=-\eta D^{-1}m$ minimizes the objective~\citep{Duchi2011AdaptiveSM}:
\begin{align*}
    \mathcal L_{\adam}(\Delta):=m^\top\Delta+\frac{1}{2\eta}\Delta^\top D\Delta\;,
\end{align*}
where $\eta$ is the learning rate.
To induce head-level sparsity, we add a group-lasso penalty and consider
\begin{align*}
    \mathcal L_{{\adam}_\lambda}(\Delta):=m^\top\Delta+\frac{1}{2\eta}\Delta^\top D\Delta+\lambda\sum_{h=1}^H\|\Delta_h\|_2\;,
    \qquad \lambda\ge0\;.
\end{align*}
\vspace{-0.1in}
\begin{proposition}
\label{prop:sparse-adam}
The unique minimizer $\Delta^\star(\lambda)$ of
$\mathcal L_{\adam_\lambda}(\Delta)$ satisfies that $\exists k$, s.t.
\begin{align*}
    \operatorname{supp}_{\rm head}(\Delta^\star(\lambda))=\underset{h}{\arg\TopK}\left(\|m_h\|_2^2\right)\;.
\end{align*}
\end{proposition}
% \begin{proposition}
% \label{prop:sparse-adam-optimal}
%     Define $\Delta_{S}^\star$ as a minimizer of $\mathcal L_{\adam}(\Delta)$ w.r.t. $\operatorname{supp}_{\rm head}(\Delta)\subseteq S$, and $S_k^{(m)}$ be the set of the $k$ heads with largest $\|m_h\|_2$. Then we have
%     \begin{align*}
%         \mathcal L_{\adam}(\Delta_{S_k^{(m)}}^\star)\le \min_{|S|\le k}\frac{\min_{i\in \theta_{S}}\sqrt{v_i}+\eps_{\adam}}{\max_{i\in \theta_{S_k^{(m)}}}\sqrt{v_i}+\eps_{\adam}}\cdot \mathcal L_{\adam}(\Delta_{S}^\star)\;.
%     \end{align*}
% \end{proposition}
\vspace{-0.1in}
\begin{remark}[Connection to practical optimization]
\Cref{prop:sparse-adam} shows that a group-sparse Adam update selects heads according to the first-moment magnitudes $\|m_h\|_2$.
% , while \Cref{prop:sparse-adam-optimal} in Appendix~\ref{sec:omitted-ranking-theory} shows that the top-$k$ heads under this ranking approximately optimize Adam's local quadratic objective. 
Since $m_h$ is the exponential moving average of past gradients, the gradient norm $\mathcal G(\{h\})=\|\nabla_{\theta_h}\mathcal L(\theta_0)\|_2$ provides a natural surrogate for $\|m_h\|_2$. \rebut{Our intention is to provide an interpretable explanation from an optimizer perspective beyond the gradient-flow analysis, showing that the gradient-norm ranking naturally arises from a standard group-lasso formulation. See empirical supports in \Cref{tab:adam-head-bound} in Appendix~\ref{sec:other_supp_tables}. We remain open to the possibility that, potentially better estimation criteria exist.}
\end{remark}
\vspace{-0.1in}
\subsection{Head-sparse learning experiments}\label{sec:ranking-exp}
\vspace{-0.1in}
\textbf{Experimental setup.}
We evaluate head-sparse RL fine-tuning on Math, Instruction-following (IF), and Coding using \lm{Qwen3-1.7B}, \lm{Qwen3-4B}, and \lm{Mistral-7B-Instruct-v0.3}. For each task, we rank attention heads by gradient norms at the initial checkpoint and compare training the top 50\% and bottom 50\% heads, together with full-attention tuning and full tuning baselines. We also include random 50\% attention-head tuning, where 50\% head blocks are randomly selected per layer, as a baseline. All training jobs share the same training configuration, except that, given the substantially larger number of trainable parameters, we use smaller learning rates for full tuning of the \lm{Qwen3} series and for experiments on \lm{Mistral-7B-Instruct-v0.3}. Implementation details are provided in Appendix~\ref{sec:implementation}.

\textbf{Main results.}
\Cref{tab:sparse-final,tab:sparse-mistral} show a consistent advantage of gradient-based head selection. Top-50\% tuning outperforms bottom-50\% tuning across all model--task pairs and evaluation splits, and also generally outperforms the random-50\% baseline. While random selection occasionally performs better on some tasks for \lm{Qwen3-1.7B}, the advantage of top-ranked heads is more robust across various settings. Despite updating only half of the attention heads, top-50\% tuning often matches or exceeds full-attention and full tuning, with the same trend extending to \lm{Mistral-7B-Instruct-v0.3}. As shown in \Cref{tab:efficiency} in Appendix~\ref{sec:supp_tables}, head-sparse tuning also reduces the gradients to be stored and synchronized at each step, offering a significant communication-efficiency advantage in distributed training~\citep{fireworks2026frontier_rl,Miahi2026UnderstandingAE}. These results support gradient norm as an effective criterion for identifying heads useful for sparse RL tuning, consistent with the theoretical claims in \Cref{sec:ranking-theory}. See supplemental results in Appendix~\ref{sec:supp_curves}.

\textbf{Comments on learning rates.} Readers may wonder whether the weaker bottom 50\% results simply arise because their smaller gradient scales require larger learning rates. We note that, under the fixed training configuration and the same trainable-parameter budget, this already reflects a nontrivial training-efficiency gap; moreover, if the difference were only a uniform gradient rescaling, Adam's normalization would cancel this effect. We further sweep the learning rate for bottom 50\% attention-head tuning in Appendix~\ref{sec:lr_tuning}. It reveals that larger learning rates help in some cases, but the gap to top 50\% attention-head tuning generally remains.

\textbf{Head-ranking consistency.}
\Cref{tab:head-order-rank-distance} in Appendix~\ref{sec:supp_tables} reports Spearman distances between independently computed head rankings. Within-task distances are consistently small, indicating stable rankings across subsets of the same task. Cross-task distances are larger, reflecting task-specific differences in head importance, but remain small overall, suggesting a shared intrinsic hierarchy across tasks. In practice, this consistency also suggests that, when the training set is large, head rankings can be estimated from a small subset of the data to reduce the computation cost.

\textbf{Layer-ranking consistency.}
We further compare gradient-norm ranking with the candidate-layer rankings obtained by direct performance evaluation. As shown in \Cref{tab:layer-importance-pred}, the two rankings are highly consistent across the \lm{Qwen3-Base} series, and exactly match for \lm{Qwen3-4B-Base}. This provides additional evidence that gradient norm can serve as a useful proxy for identifying trainable attention components without exhaustively evaluating each candidate. The layer distributions of top heads are visualized in Appendix~\ref{sec:layer-dist}. \rebut{We further validate the effectiveness of within-layer gradient-based head selection through a controlled ablation in \Cref{tab:per-layer-if} in Appendix~\ref{sec:ablation}.}

\textbf{Sparsity--performance trade-off.} Beyond the fixed 50\% sparsity setting, we further vary the number of trainable heads to examine how performance changes with the head budget. All configurations use the same training setup, differing only in the number of trainable attention heads. \Cref{fig:budget-sweep} shows that performance generally improves as more top-ranked heads are included, while the major performance can already be achieved with a relatively small head budget.

\begin{table}[H]
\vspace{-0.2in}
\caption{\rebut{Sparse RL accuracy (\%) of \lm{Qwen3} series. Parenthesized arrows are the performance drops compared with top 50\%. The top-2 results, including ties, are highlighted in pink. Top/bottom 50\% and full-attn performance is reported as mean$\pm$std over three random rollout seeds $\{32,42,52\}$.}}
\label{tab:sparse-final}
\centering
\small
\setlength{\tabcolsep}{3.5pt}
\begin{tabular}{lllcccccc}
\toprule
Model & Task & Eval split & Base & Top 50\% & Bottom 50\% & Rand 50\% & Full-attn & Full \\
\midrule
\lm{Qwen3-1.7B}
 & Math & In-dist test
 & \cellcolor[gray]{0.88} 26.6
 & \cellcolor[RGB]{255,225,235} 36.7 {\scriptsize $\pm 1.0$}
 & 30.9 {\scriptsize $\pm 1.0$} ($\downarrow 5.8$)
 & \cellcolor[RGB]{255,225,235} 38.6
 & 36.0 {\scriptsize $\pm 2.5$}
 & 35.3 \\
 & & Out-dist test
 & \cellcolor[gray]{0.88} 70.2
 & \cellcolor[RGB]{255,225,235} 76.7 {\scriptsize $\pm 0.9$}
 & 73.6 {\scriptsize $\pm 1.4$} ($\downarrow 3.1$)
 & \cellcolor[RGB]{255,225,235} 79.2
 & \cellcolor[RGB]{255,225,235} 76.7 {\scriptsize $\pm 2.2$}
 & 73.6 \\
\addlinespace[1pt]

\lm{Qwen3-1.7B}
 & IF & In-dist test
 & \cellcolor[gray]{0.88} 41.7
 & \cellcolor[RGB]{255,225,235} 57.3 {\scriptsize $\pm 0.8$}
 & 47.0 {\scriptsize $\pm 0.2$} ($\downarrow 10.3$)
 & 51.6
 & \cellcolor[RGB]{255,225,235} 60.2 {\scriptsize $\pm 1.3$}
 & 54.6 \\
 & & Out-dist test
 & \cellcolor[gray]{0.88} 17.7
 & \cellcolor[RGB]{255,225,235} 22.2 {\scriptsize $\pm 2.7$}
 & 15.9 {\scriptsize $\pm 0.8$} ($\downarrow 6.3$)
 & 19.7
 & \cellcolor[RGB]{255,225,235} 22.9 {\scriptsize $\pm 4.7$}
 & 19.7 \\
\addlinespace[1pt]

\lm{Qwen3-1.7B}
 & Coding & In-dist test
 & \cellcolor[gray]{0.88} 53.7
 & \cellcolor[RGB]{255,225,235} 59.5 {\scriptsize $\pm 1.2$}
 & 57.8 {\scriptsize $\pm 0.8$} ($\downarrow 1.8$)
 & 58.5
 & 53.9 {\scriptsize $\pm 10.9$}
 & \cellcolor[RGB]{255,225,235} 61.6 \\
 & & Out-dist test
 & \cellcolor[gray]{0.88} 33.6
 & \cellcolor[RGB]{255,225,235} 35.2 {\scriptsize $\pm 1.4$}
 & 33.9 {\scriptsize $\pm 0.4$} ($\downarrow 1.3$)
 & \cellcolor[RGB]{255,225,235} 35.6
 & 33.3 {\scriptsize $\pm 6.1$}
 & 34.4 \\
\addlinespace[1pt]

\lm{Qwen3-4B}
 & Math & In-dist test
 & \cellcolor[gray]{0.88} 42.4
 & \cellcolor[RGB]{255,225,235} 56.9 {\scriptsize $\pm 2.8$}
 & 50.8 {\scriptsize $\pm 2.5$} ($\downarrow 6.2$)
 & 53.3
 & 53.5 {\scriptsize $\pm 8.6$}
 & \cellcolor[RGB]{255,225,235} 58.2 \\
 & & Out-dist test
 & \cellcolor[gray]{0.88} 81.4
 & \cellcolor[RGB]{255,225,235} 86.6 {\scriptsize $\pm 1.9$}
 & 84.1 {\scriptsize $\pm 1.2$} ($\downarrow 2.5$)
 & 83.6
 & 85.1 {\scriptsize $\pm 3.1$}
 & \cellcolor[RGB]{255,225,235} 88.2 \\
\addlinespace[1pt]

\lm{Qwen3-4B}
 & IF & In-dist test
 & \cellcolor[gray]{0.88} 50.5
 & \cellcolor[RGB]{255,225,235} 60.7 {\scriptsize $\pm 3.8$}
 & 52.5 {\scriptsize $\pm 1.0$} ($\downarrow 8.2$)
 & 56.3
 & \cellcolor[RGB]{255,225,235} 63.9 {\scriptsize $\pm 7.4$}
 & 59.2 \\
 & & Out-dist test
 & \cellcolor[gray]{0.88} 21.0
 & \cellcolor[RGB]{255,225,235} 25.6 {\scriptsize $\pm 1.3$}
 & 25.4 {\scriptsize $\pm 1.0$} ($\downarrow 0.1$)
 & 25.3
 & \cellcolor[RGB]{255,225,235} 27.7 {\scriptsize $\pm 2.3$}
 & 24.0 \\
\addlinespace[1pt]

\lm{Qwen3-4B}
 & Coding & In-dist test
 & \cellcolor[gray]{0.88} 64.6
 & \cellcolor[RGB]{255,225,235} 69.8 {\scriptsize $\pm 1.6$}
 & 67.1 {\scriptsize $\pm 0.7$} ($\downarrow 2.7$)
 & 66.7
 & \cellcolor[RGB]{255,225,235} 68.8 {\scriptsize $\pm 3.2$}
 & 66.4 \\
 & & Out-dist test
 & \cellcolor[gray]{0.88} 46.3
 & \cellcolor[RGB]{255,225,235} 53.0 {\scriptsize $\pm 2.2$}
 & 47.0 {\scriptsize $\pm 3.4$} ($\downarrow 5.9$)
 & 37.9
 & \cellcolor[RGB]{255,225,235} 50.8 {\scriptsize $\pm 2.6$}
 & 43.6 \\
\bottomrule
\end{tabular}
\end{table}
\begin{table}[H]
\vspace{-0.2in}
\caption{Sparse RL accuracy (\%) of \lm{Mistral-7B-Instruct-v0.3}.   Parenthesized arrows are the performance drops compared with top 50\%. The highest results are highlighted in pink.}
\label{tab:sparse-mistral}
\centering
\small
\setlength{\tabcolsep}{3.5pt}
\begin{tabular}{lllcccc}
\toprule
Model & Task & Eval split & Base & Top 50\% & Bottom 50\% & Full-attn \\
\midrule
\lm{Mistral-7B-Instruct-v0.3} & IF & In-dist test & \cellcolor[gray]{0.88} 27.4 & \cellcolor[RGB]{255,225,235} 69.9 & 53.5 ($\downarrow 16.4$) & 56.9  \\
 &  & Out-dist test & \cellcolor[gray]{0.88} 18.3 & 21.0 & 20.3 ($\downarrow 0.7$) & \cellcolor[RGB]{255,225,235} 25.3  \\
\bottomrule
\end{tabular}
\end{table}
\begin{table}[H]
\centering
\small
\vspace{-0.2in}
\caption{Math candidate-layer rankings for \lm{Qwen3-Base} series. Rankings are obtained using the gradient norms of attention heads in each layer, computed on a $3$k-subset of \texttt{DAPO-Math-17k}.}
\label{tab:layer-importance-pred}
\begin{tabular}{l l l}
\toprule
Model & Performance ranking~\citep{Zhang2026IsOL} & Gradient norm ranking \\
\midrule
\lm{Qwen3-1.7B-Base} &
$10 > 12 > 1 > 7 > 24$ &
$12 > 10 > 7 > 24 > 1$ \\
\lm{Qwen3-4B-Base} &
$16 > 14 > 11 > 24 > 2$ &
$16 > 14 > 11 > 24 > 2$ \\
\lm{Qwen3-8B-Base} &
$16 > 15 > 8 > 2 > 0$ &
$15 \rebut{\approx} 16 > 0 > 8 > 2$ \\
\bottomrule
\end{tabular}
\end{table}

\begin{figure}[htbp]
\centering
\begin{minipage}[htbp]{0.45\textwidth}\centering
\includegraphics[width=\linewidth]{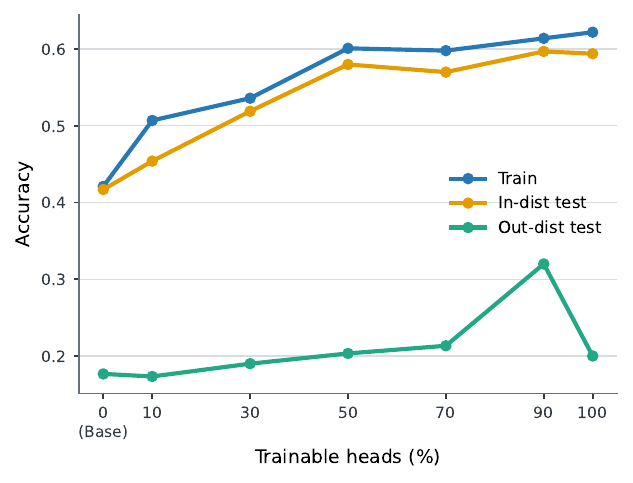}\\[-0.2em]
{\small (a) \lm{Qwen3-1.7B}, IF}
\end{minipage}\hfill
\begin{minipage}[htbp]{0.45\textwidth}\centering
\includegraphics[width=\linewidth]{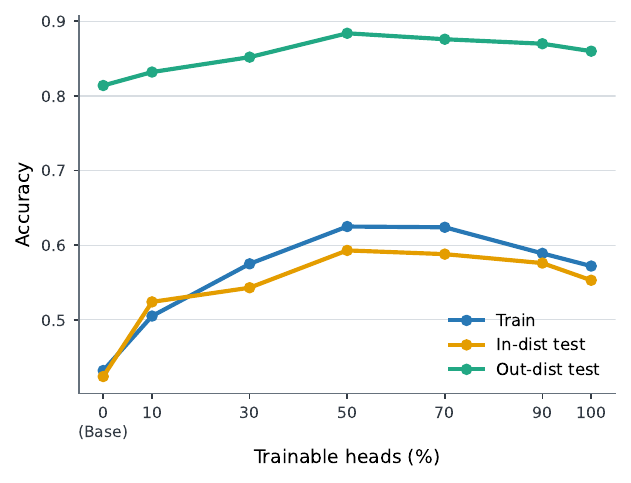}\\[-0.2em]
{\small (b) \lm{Qwen3-4B}, Math}
\end{minipage}
\caption{Sparsity--performance trade-off on training, in-distribution test, and out-of-distribution test sets. Each panel includes top 10, 30, 50, 70, and 90\% attention-head tuning plus full-attention tuning. All jobs share exactly the same training configuration.}
\label{fig:budget-sweep}
\end{figure}
\section{Partitioning the Attention Heads}
\subsection{Theoretical %understandings
\rebut{explanation}}\label{sec:partition-theory}
We consider sequential learning of two tasks $A$ and $B$. Starting from a base model $\theta_0$, task $A$ is trained on a head set $S_A$ to obtain $\theta_A$, after which task $B$ is trained on a head set $S_B$ from $\theta(0)=\theta_A$. Recall that \Cref{sec:ranking} motivates using high-gradient-norm heads for task $A$. 

Updating the same parameters for task $B$ can overwrite the parameter changes induced by task $A$. This concern is particularly pronounced when the trainable parameters are restricted to a sparse subset, motivating the use of disjoint head sets across tasks. Sometimes we may assume an exact partition, $S_B\subseteq S_A^c$.

We now characterize the trade-off between learning task $B$ and preserving task $A$, \textit{i.e.}, plasticity and stability. Omitted proofs are deferred to Appendix~\ref{sec:omitted-partition-theory}.

\begin{proposition}
\label{prop:continual-gradient-flow}
Let $\theta(t)$ follow task-$B$ Euclidean gradient flow restricted to $S_B$, initialized at $\theta(0)=\theta_A$. Then $-\frac{d}{dt}\mathcal L_B(\theta(t))=\mathcal G_B(S_B;\theta(t))^2$, while the change in the task-$A$ loss satisfies $\left|\frac{d}{dt}\mathcal L_A(\theta(t))\right|\le\mathcal G_A(S_B;\theta(t))\mathcal G_B(S_B;\theta(t))$.
\end{proposition}

\begin{proof}
The first equation is by \Cref{prop:gradnorm_flow_eq}. The gradient flow satisfies $\dot\theta_{S_B}(t)=-\nabla_{\theta_{S_B}}\mathcal L_B(\theta(t))$ and $\dot\theta_{S_{B}^c}(t)=0$, so the chain rule gives $\frac{d}{dt}\mathcal L_A(\theta(t))=-\langle\nabla_{\theta_{S_B}}\mathcal L_A(\theta(t)),\nabla_{\theta_{S_B}}\mathcal L_B(\theta(t))\rangle$.
The desired bound follows from Cauchy--Schwarz inequality.
\end{proof}

\begin{corollary}
\label{cor:gradnorm-forgetting-average}
Define the contraharmonic time-averaged gradient norm by
\begin{align*}
    \widetilde{\mathcal G}_B(S_B;t):=\frac{\int_0^t\mathcal G_B(S_B;\theta(s))^2\,ds}{\int_0^t\mathcal G_B(S_B;\theta(s))\,ds}\;,\quad \widetilde{\mathcal G}_{A\mid B}(S_B;t):=\frac{\int_0^t\mathcal G_A(S_B;\theta(s))\mathcal G_B(S_B;\theta(s))\,ds}{\int_0^t\mathcal G_B(S_B;\theta(s))\,ds}\;.
\end{align*}
Then $\left|\mathcal L_A(\theta(t))-\mathcal L_A(\theta_A)\right|\le\frac{\widetilde{\mathcal G}_{A\mid B}(S_B;t)}{\widetilde{\mathcal G}_B(S_B;t)}[\mathcal L_B(\theta_A)-\mathcal L_B(\theta(t))]$.
\end{corollary}
\begin{remark}
For task $B$, we prefer heads with small task-$A$ but large task-$B$ gradient norm. \rebut{The ranking of task-$A$ gradient norm is stable during old-task training} when the gradient changes smoothly \rebut{(see empirical evidence in \Cref{tab:ranking-consistency} in Appendix~\ref{sec:other_supp_tables})}. This suggests balancing the allocation by assigning some high-task-$A$-gradient heads to task $A$ and some high-task-$B$-gradient heads to task $B$, motivating our later allocation strategy design.
\end{remark}
\begin{assumption}\label{assum:continual_smooth}
\rebut{Suppose the task-$B$ restricted gradient flow exists on $[0,T]$, and $\exists\beta_A,\beta_B\ge0$ s.t.
\begin{align*}
    -\frac{\nabla_{\theta_{S_B}}\mathcal L_A(\theta(t))^\top\nabla_{\theta_{S_B}}^2\mathcal L_A(\theta(t))\nabla_{\theta_{S_B}}\mathcal L_B(\theta(t))}{\|\nabla_{\theta_{S_B}}\mathcal L_A(\theta(t))\|_2\|\nabla_{\theta_{S_B}}\mathcal L_B(\theta(t))\|_2}\le\beta_A\;,\;\frac{\nabla_{\theta_{S_B}}\mathcal L_B(\theta(t))^\top\nabla_{\theta_{S_B}}^2\mathcal L_B(\theta(t))\nabla_{\theta_{S_B}}\mathcal L_B(\theta(t))}{\|\nabla_{\theta_{S_B}}\mathcal L_B(\theta(t))\|_2^2}\le\beta_B\;,
\end{align*}
whenever the corresponding gradient norms $>0$.}
\end{assumption}
\begin{proposition}
\label{prop:gradnorm-forgetting}
\rebut{Under \Cref{assum:continual_smooth},} for every $t\in(0,T]$,
\begin{align*}
    \left|\mathcal L_A(\theta(t))-\mathcal L_A(\theta_A)\right|\le\frac{{\mathcal G}_{A}(S_B;\theta_A)}{\widetilde{\mathcal G}_B(S_B;t)}[\mathcal L_B(\theta_A)-\mathcal L_B(\theta(t))]+\frac{\beta_A\left[\mathcal L_B(\theta_A)-\mathcal L_B(\theta(t))\right]^2}{2\widetilde{\mathcal G}_B(S_B;t)^2}\;,
\end{align*}
where $\widetilde{\mathcal G}_B(S_B;t)\ge\sqrt{\frac{1-e^{-2\beta_Bt}}{2\beta_Bt}}\,\mathcal G_B(S_B;\theta_A)$.
\end{proposition}
\begin{remark}
Similar to \Cref{prop:gradnorm_flow}, this result demonstrates that the initial gradient norms can remain informative for controlling the stability--plasticity trade-off over a nontrivial training interval. \rebut{The typical values of $\beta_A,\beta_B$ are reported in \Cref{tab:continual-second-order} in Appendix~\ref{sec:other_supp_tables}.} In our experiments in \Cref{sec:partition-exp}, rankings computed at the initial checkpoint effectively guide continual training over hundreds of training steps. Also, see \Cref{sec:attention-curvature} for discussions on \Cref{assum:continual_smooth}.
\end{remark}

\subsection{Sequential continual learning experiments}\label{sec:partition-exp}

\paragraph{Experimental setup.}
We evaluate whether partitioning attention heads can reduce interference in continual RL. We sequentially fine-tune \lm{Qwen3-1.7B} and \lm{Qwen3-4B} on two- and three-task sequences over Math, IF, and Coding. We compare head-partitioned tuning with continual full-attention and full-tuning baselines. Implementation details are provided in Appendix~\ref{sec:implementation}. \rebut{A partition v.s. non-partition ablation study is presented in \Cref{tab:partition-vs-nonpartition,tab:partition-vs-reuse-full} in Appendix~\ref{sec:ablation}.}

Motivated by \Cref{sec:partition-theory}, we consider three head-allocation strategies:
\begin{itemize}[leftmargin=*, itemsep=1pt, topsep=2pt]
\item \textbf{Top--Bottom}: the previous task uses top-ranked heads, while the next task uses the remaining.
\item \textbf{Round-Robin}: tasks alternately select their highest-ranked unassigned heads.
\item \textbf{Shared}: in addition to disjoint task-specific heads, a remaining subset is shared across tasks.
\end{itemize}
For two-task experiments, we compare the first two strategies. For three-task experiments, we use the third, assigning 25\% of the heads to each task and sharing the remaining 25\%. It is worth emphasizing that we do not claim these strategies to be optimal and exhaustive. Instead, they serve to validate our theory and demonstrate possible ways to apply it.

\paragraph{Main results.} \Cref{fig:continual-qwen317b-math-mbpp,fig:continual-qwen34b-mbpp-ifeval,fig:continual-qwen34b-math-ifeval} show that head partitioning generally reduces forgetting while preserving adaptation to the new task, especially when full-attention tuning causes large drops on earlier tasks. Top--Bottom typically retains the first task better but adapts slightly worse to the second, while Round-Robin shows the opposite trade-off, consistent with their allocation strategies. \Cref{fig:continual-qwen317b-3task,fig:continual-qwen34b-3task} further show that the head-partition approach extends to three sequential tasks with minimal degradation on previous tasks, supporting the theoretical claims in \Cref{sec:partition-theory}. See supplemental results in Appendix~\ref{sec:supp_tables}.

\begin{figure}[H]
\centering
\includegraphics[width=0.9\textwidth]{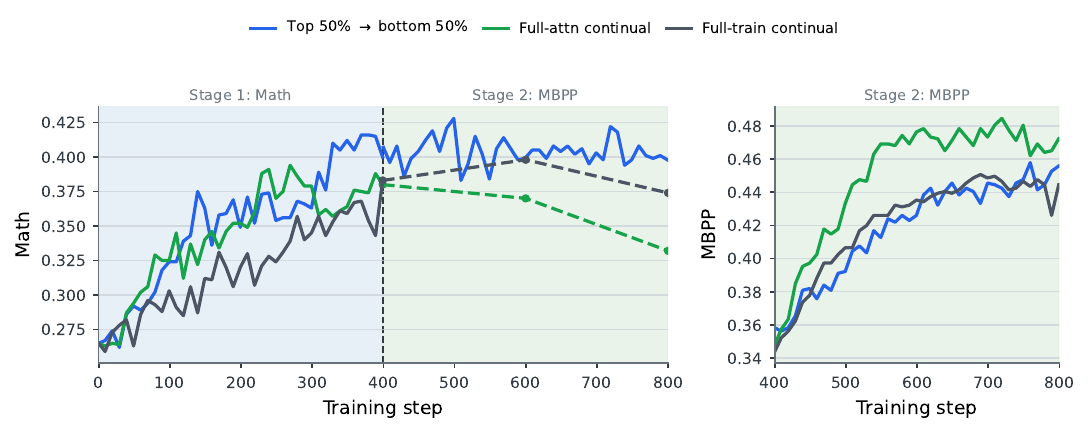}
\caption{Continual RL accuracy of \lm{Qwen3-1.7B} (Math $\rightarrow$ Coding). The model is first trained on Math task and then on Coding task. Shaded regions mark stages.}
\label{fig:continual-qwen317b-math-mbpp}
\end{figure}

\begin{figure}[H]
\centering
\includegraphics[width=0.9\textwidth]{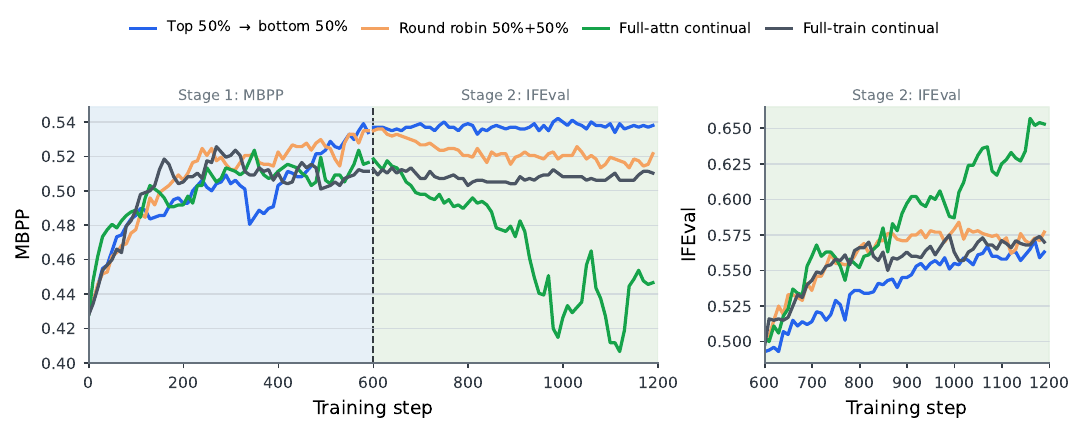}
\vspace{-0.2in}
\caption{Continual RL accuracy of \lm{Qwen3-4B} (Coding $\rightarrow$ IF). The model is first trained on Coding task and then on IF task. Shaded regions mark stages.}
\label{fig:continual-qwen34b-mbpp-ifeval}
\end{figure}

\begin{figure}[H]
\vspace{-0.2in}
\centering
\includegraphics[width=0.9\textwidth]{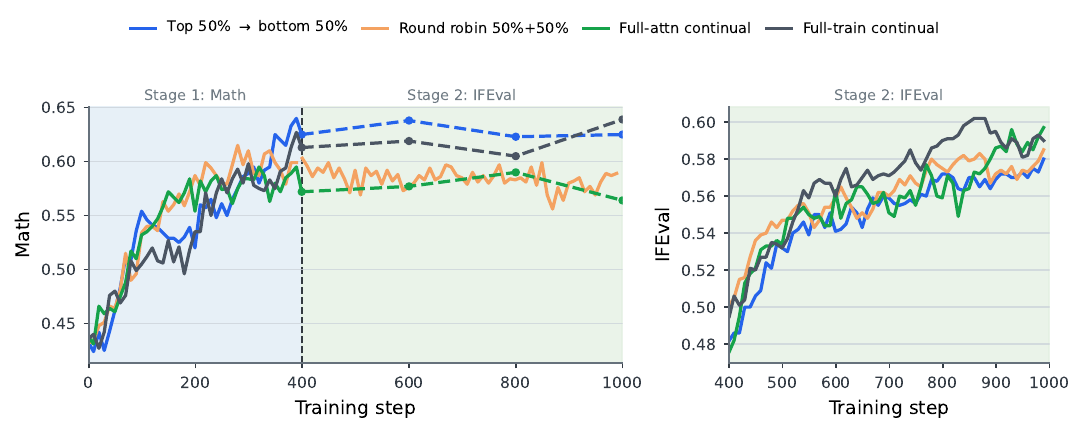}
\vspace{-0.2in}
\caption{Continual RL accuracy of \lm{Qwen3-4B} (Math $\rightarrow$ IF). The model is first trained on Math task and then on IF task. Shaded regions mark stages.}
\label{fig:continual-qwen34b-math-ifeval}
\end{figure}

\begin{figure}[h]
\vspace{-0.2in}
\centering
\includegraphics[width=0.9\textwidth]{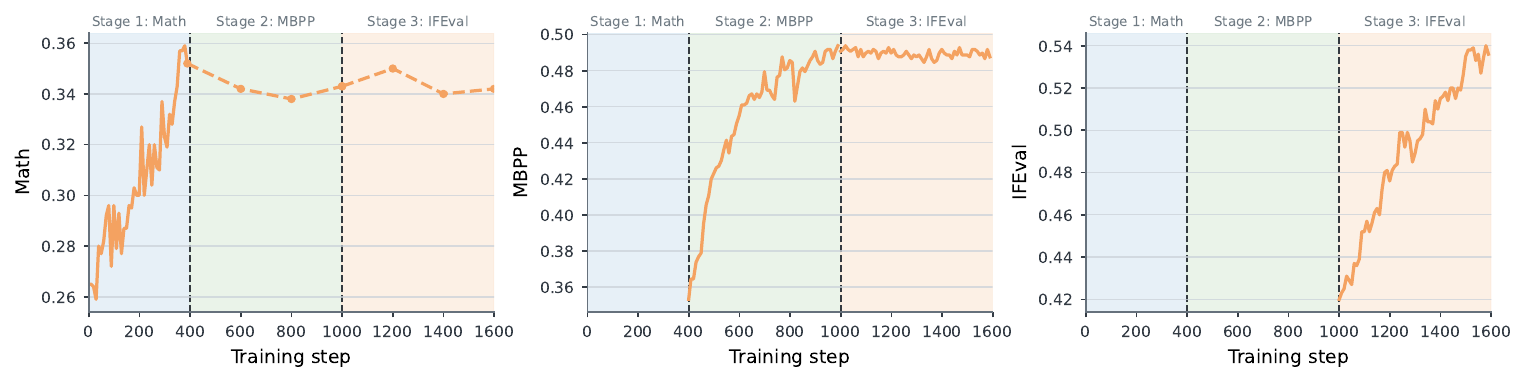}
\vspace{-0.2in}
\caption{Continual RL accuracy of \lm{Qwen3-1.7B} (Math $\rightarrow$ Coding $\rightarrow$ IF). The model is first trained on Math task, then on Coding task, and finally on IF task. Shaded regions mark stages.}
\label{fig:continual-qwen317b-3task}
\end{figure}

\begin{figure}[h]
\vspace{-0.2in}
\centering
\includegraphics[width=0.9\textwidth]{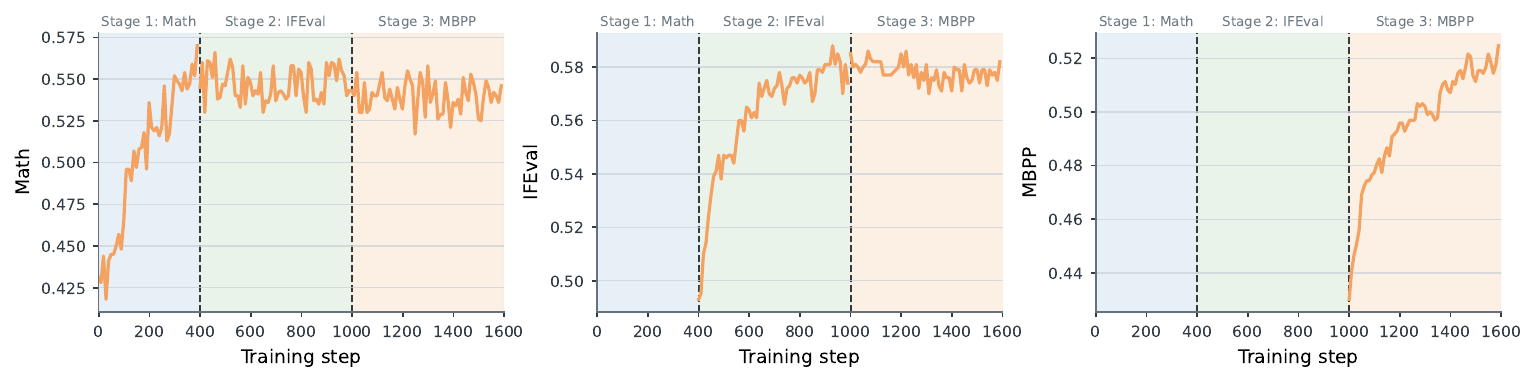}
\vspace{-0.2in}
\caption{Continual RL accuracy of \lm{Qwen3-4B} (Math $\rightarrow$ IF $\rightarrow$ Coding). The model is first trained on Math task, then on IF task, and finally on Coding task. Shaded regions mark stages.}
\label{fig:continual-qwen34b-3task}
\end{figure}

Furthermore, to test whether the learned adaptations are functionally localized, we reset the stage-$\mathrm I$ heads to their original values after stage $\mathrm {II}$. As shown in \Cref{tab:stage1-head-reset} in Appendix~\ref{sec:supp_tables}, this operation largely degrades stage-$\mathrm I$ performance while maintaining stage-$\mathrm{II}$ performance, suggesting approximate functional isolation across the allocated head subsets. This agrees with the observation of \citet{jin2025finetuning} that attention-only fine-tuning is well approximated by model linearization.

\section{Second-Order Structure of Sparse Attention-Head Updates}
\label{sec:attention-curvature}
In \Cref{sec:ranking-theory,sec:partition-theory}, we show that favorable second-order properties allow us to control how gradient norms evolve during fine-tuning. For $S\subseteq U$, $\|\nabla_{\theta_S}^2\mathcal L\|_{\mathrm{op}}\le\|\nabla_{\theta_U}^2\mathcal L\|_{\mathrm{op}}$, so sparsity generically controls second-order effects but does not explain why attention heads are preferable to arbitrary sparse parameter subsets. We next develop a more detailed, attention-head-specific second-order analysis. Omitted proofs are deferred to Appendix~\ref{sec:proof-of-attention-curvature}.

We consider a single multi-head self-attention layer on a fixed rollout batch and study the second-order bound w.r.t. its attention parameters, treating all subsequent computation as part of the loss. Let $F(\theta)$ denote the layer output and let $\mathcal L(\theta)=\ell(F(\theta))$. Following \citet{ormaniec2024}, we use the Gauss-Newton decomposition of the Hessian:
\vspace{-0.05in}
\begin{align*}
\mathbf H:=\nabla_\theta^2\mathcal L=\mathbf H_f+\mathbf H_o\;,\;
\mathbf H_f:=\frac{\partial\ell}{\partial F}\nabla_\theta^2F\;,\;
\mathbf H_o:=\left(\frac{\partial F}{\partial\theta}\right)^\top\nabla_F^2\ell(F)\left(\frac{\partial F}{\partial\theta}\right)\;.
\end{align*}
\vspace{-0.05in}
% \ruizhe{model distribution; expectation $H_f$; seems the expectation $H_f$ is not zero here because we are using policy gradient? i'm not sure}
\begin{proposition}
\label{prop:head-local-functional-curvature}
For multi-head self-attention and two heads $h\neq h'$, $(\mathbf H_f)_{hh'}=0$. Consequently, for any head set $S$, $\|(\mathbf H_f)_{SS}\|_{\mathrm{op}}=\max_{h\in S}\|(\mathbf H_f)_{hh}\|_{\mathrm{op}}$.
\end{proposition}
\begin{remark}
Attention heads are naturally isolated as each head depends only on its own parameter block, making $H_f$ block-diagonal across heads. This does not hold for shared parameterizations such as standard LoRA, where the matrix $A$ is shared across attention heads, nor for more unstructured sparse parameter subsets. 
\end{remark}

We can further show that when attention concentrates on similar token representations, a sharper single-head bound is possible. Such a regime is consistent with the selective and specialized attention patterns observed empirically~\citep{yin2025which,Olsson2022IncontextLA}.

\begin{proposition}[Special case: concentrated attention]
\label{prop:concentrated-attention-curvature}
Consider one attention head $h$. Suppose that for every attention row $A_{i,:}$ there exists a token set $T_i$ s.t. $\sum_{j\in T_i}A_{ij}\ge1-\varepsilon$ and $\max_{j,j'\in T_i}\|X_{j,:}-X_{j',:}\|_2\le\rho$.
Assume the input and weight norms are bounded and $\|\nabla_F\ell(F)\|_2\le G$. Then we have
\begin{align*}
\|(\mathbf H_f)_{hh}\|_{\mathrm{op}}=O\left(G\left(\rho^2+\rho^3+\varepsilon\right)\right)\;.
\end{align*}
\end{proposition}

\begin{proposition}
\label{prop:outer-hessian-entk}
Define the tangent kernel $K_S:=\frac{\partial F}{\partial\theta_S}\left(\frac{\partial F}{\partial\theta_S}\right)^\top$ for a head set $S$.
If the outer objective is locally $\beta$-smooth in the attention output, \textit{i.e.}, $\|\nabla_F^2\ell(F)\|_{\mathrm{op}}\le\beta$,
then
\begin{align*}
\|(\mathbf H_o)_{SS}\|_{\mathrm{op}}\le\beta\|K_S\|_{\mathrm{op}}\;.
\end{align*}
\end{proposition}

\begin{remark}
    \citet{jin2025finetuning} \rebut{empirically observe} that attention-only fine-tuning exhibits kernel-like behavior, \rebut{supporting} a locally linearized view \rebut{of attention-head updates}. \rebut{This provides empirical motivation for considering a regime with moderate Jacobian variation, in which} $K_S$ remains close to its initial value and \rebut{$\|K_S\|_{\mathrm{op}}$ remains stable. The boundedness and local smoothness assumptions in \Cref{prop:concentrated-attention-curvature,prop:outer-hessian-entk} are standard regularity conditions.}
\end{remark}

\begin{corollary}
\label{cor:attention-second-order}
Under the assumptions of \Cref{prop:outer-hessian-entk}, for any head set $S$ and any vector $v_S\neq0$, whenever $\nabla_{\theta_S}\mathcal L\neq0$, we have
\begin{align*}
\frac{\left|(\nabla_{\theta_S}\mathcal L)^\top\mathbf H_{SS}v_S\right|}{\|\nabla_{\theta_S}\mathcal L\|_2\|v_S\|_2}\le\max_{h\in S}\|(\mathbf H_f)_{hh}\|_{\mathrm{op}}+\beta\|K_S\|_{\mathrm{op}}\;.
\end{align*}
\end{corollary}
This result recovers \rebut{\Cref{assum:sparse_smooth,assum:continual_smooth}} used in \Cref{prop:gradnorm_flow,prop:gradnorm-forgetting}, explaining why attention-head tuning can enjoy a favorable second-order bound. \rebut{See empirical values in Appendix~\ref{sec:other_supp_tables}.}
\vspace{-0.1in}
\section{Conclusion}
\vspace{-0.1in}
We study sparse and continual RL fine-tuning through gradient-based ranking and partitioning of attention heads. We show theoretically that gradient norm characterizes the training benefit of a head subset. Empirically, high-ranked heads support efficient sparse RL tuning, while partitioning heads across tasks effectively reduces forgetting in continual RL.

It is also important to acknowledge our limitations. (i) Our experiments are restricted to relatively small language models, limited tasks, and settings without long reasoning traces. (ii) Our analysis focuses on attention parameters and may not directly extend to MLP layers, which would be important for sparse full-model fine-tuning. (iii) The theoretical second-order analysis is restricted to a single multi-head self-attention layer.
\newpage
\subsection*{Acknowledgement}
MF was supported in part by awards NSF TRIPODS II DMS-2023166, CCF 2212261, CCF 2312775, and the Moorthy Professorship at the UW.
\subsection*{AI use statement}
In this work, we used generative AI tools for paper drafting and polishing, coding, literature review, and quickly testing mathematical intuitions. We have not used generative AI tools for generating synthetic datasets. LLM-generated code was tested for correctness. All final theoretical results are of our own.  All AI outputs were reviewed and edited by the authors.
We take responsibility for the final content of this work, including text, claims or artifacts produced with the aid of generative AI.
\bibliography{ref}

@misc{sagitova2026,
      title={Specialization of softmax attention heads: insights from the high-dimensional single-location model}, 
      author={M. Sagitova and O. Duranthon and L. Zdeborová},
      year={2026},
      eprint={2603.03993},
      archivePrefix={arXiv},
      primaryClass={cs.LG},
      url={https://arxiv.org/abs/2603.03993}, 
}

@inproceedings{
    ormaniec2024,
    title={What Does It Mean to Be a Transformer? Insights from a Theoretical Hessian Analysis},
    author={Weronika Ormaniec and Felix Dangel and Sidak Pal Singh},
    booktitle={The Thirteenth International Conference on Learning Representations},
    year={2025},
    url={https://openreview.net/forum?id=3ddi7Uss2A}
}

@inproceedings{Dan2019OptimalAO,
    title={Optimal Analysis of Subset-Selection Based L\_p Low Rank Approximation},
    author={Chen Dan and Hong Wang and Hongyang Zhang and Yuchen Zhou and Pradeep Kumar Ravikumar},
    booktitle={Neural Information Processing Systems},
    year={2019},
    url={https://api.semanticscholar.org/CorpusID:260537614}
}

@inproceedings{
    jin2025finetuning,
    title={Fine-Tuning Attention Modules Only: Enhancing Weight Disentanglement in Task Arithmetic},
    author={Ruochen Jin and Bojian Hou and Jiancong Xiao and Weijie J Su and Li Shen},
    booktitle={The Thirteenth International Conference on Learning Representations},
    year={2025},
    url={https://openreview.net/forum?id=dj0TktJcVI}
}

@article{Zhang2026IsOL,
    title={Is One Layer Enough? Training A Single Transformer Layer Can Match Full-Parameter RL Training},
    author={Zijian Zhang and Rizhen Hu and Athanasios Glentis and Dawei Li and Chung-Yiu Yau and Hongzhou Lin and Mingyi Hong},
    journal={ArXiv},
    year={2026},
    volume={abs/2607.01232},
    url={https://api.semanticscholar.org/CorpusID:289706488}
}

@article{Miahi2026UnderstandingAE,
    title={Understanding and Exploiting Weight Update Sparsity for Communication-Efficient Distributed RL},
    author={Erfan Miahi and Eugene Belilovsky},
    journal={ArXiv},
    year={2026},
    volume={abs/2602.03839},
    url={https://api.semanticscholar.org/CorpusID:285275603}
}

@online{fireworks2026frontier_rl,
    author       = {{Fireworks.ai}},
    title        = {Frontier RL is Cheaper Than You Think},
    howpublished = {Fireworks.ai Blog},
    year         = {2026},
    month        = {3},
    day          = {23},
    url          = {https://fireworks.ai/blog/frontier-rl-is-cheaper-than-you-think},
    note         = {Accessed: 2026-09-20}
}

@inproceedings{farajtabar2020orthogonal,
  author={Mehrdad Farajtabar and Navid Azizan and Alex Mott and Ang Li},
  title={Orthogonal Gradient Descent for Continual Learning},
  year={2020},
  cdate={1577836800000},
  pages={3762-3773},
  url={http://proceedings.mlr.press/v108/farajtabar20a.html},
  booktitle={AISTATS}
}

@inproceedings{chaudhry2020continual,
    author={Arslan Chaudhry and Naeemullah Khan and Puneet K. Dokania and Philip H. S. Torr},
    title={Continual Learning in Low-rank Orthogonal Subspaces},
    year={2020},
    cdate={1577836800000},
    url={https://proceedings.neurips.cc/paper/2020/hash/70d85f35a1fdc0ab701ff78779306407-Abstract.html},
    booktitle={NeurIPS}
}

@inproceedings{
    saha2021gradient,
    title={Gradient Projection Memory for Continual Learning},
    author={Gobinda Saha and Isha Garg and Kaushik Roy},
    booktitle={International Conference on Learning Representations},
    year={2021},
    url={https://openreview.net/forum?id=3AOj0RCNC2}
}

@inproceedings{doan2021theoretical,
    title={A Theoretical Analysis of Catastrophic Forgetting through the NTK Overlap Matrix},
    author={Thang Van Doan and Mehdi Abbana Bennani and Bogdan Mazoure and Guillaume Rabusseau and Pierre Alquier},
    booktitle={International Conference on Artificial Intelligence and Statistics},
    year={2020},
    url={https://api.semanticscholar.org/CorpusID:222208953}
}

@inproceedings{wang2023orthogonal,
    author={Xiao Wang and Tianze Chen and Qiming Ge and Han Xia and Rong Bao and Rui Zheng and Qi Zhang and Tao Gui and Xuanjing Huang},
    title={Orthogonal Subspace Learning for Language Model Continual Learning},
    year={2023},
    cdate={1672531200000},
    pages={10658-10671},
    url={https://doi.org/10.18653/v1/2023.findings-emnlp.715},
    booktitle={EMNLP (Findings)}
}

@article{liang2024inflora,
    title={InfLoRA: Interference-Free Low-Rank Adaptation for Continual Learning},
    author={Yan-Shuo Liang and Wu-Jun Li},
    journal={2024 IEEE/CVF Conference on Computer Vision and Pattern Recognition (CVPR)},
    year={2024},
    pages={23638-23647},
    url={https://api.semanticscholar.org/CorpusID:268819832}
}

@inproceedings{
    zhang2025lori,
    title={Lo{RI}: Reducing Cross-Task Interference in Multi-Task Low-Rank Adaptation},
    author={Juzheng Zhang and Jiacheng You and Ashwinee Panda and Tom Goldstein},
    booktitle={Second Conference on Language Modeling},
    year={2025},
    url={https://openreview.net/forum?id=b8cW86QcOD}
}

@inproceedings{
    qiu2026splitlora,
    title={SplitLo{RA}: Balancing Stability and Plasticity in Continual Learning Through Gradient Space Splitting},
    author={Haomiao Qiu and Miao Zhang and Ziyue Qiao and Weili Guan and Min Zhang and Liqiang Nie},
    booktitle={The Fourteenth International Conference on Learning Representations},
    year={2026},
    url={https://openreview.net/forum?id=Zm1hjXxRQV}
}

@inproceedings{
    iurada2025efficient,
    title={Efficient Model Editing with Task-Localized Sparse Fine-tuning},
    author={Leonardo Iurada and Marco Ciccone and Tatiana Tommasi},
    booktitle={The Thirteenth International Conference on Learning Representations},
    year={2025},
    url={https://openreview.net/forum?id=TDyE2iuvyc}
}

@inproceedings{
    nayak2026sculpting,
    title={Sculpting Subspaces: Constrained Full Fine-Tuning in {LLM}s for Continual Learning},
    author={Nikhil Shivakumar Nayak and Krishnateja Killamsetty and Ligong Han and Abhishek Bhandwaldar and Prateek Chanda and Kai Xu and Oleg Silkin and Mustafa Eyceoz and Hao Wang and Aldo Pareja and Akash Srivastava},
    booktitle={The Fourteenth International Conference on Learning Representations},
    year={2026},
    url={https://openreview.net/forum?id=vQcyqsGJDw}
}

@inproceedings{
    qiao2026merge,
    title={Merge before Forget: A Single Lo{RA} Continual Learning via Continual Merging},
    author={Fuli Qiao and Mehrdad Mahdavi},
    booktitle={The Fourteenth International Conference on Learning Representations},
    year={2026},
    url={https://openreview.net/forum?id=i1Rj7yU6eF}
}

@inproceedings{
    tang2025merging,
    title={Merging on the Fly Without Retraining: A Sequential Approach to Scalable Continual Model Merging},
    author={Anke Tang and Enneng Yang and Li Shen and Yong Luo and Han Hu and Lefei Zhang and Bo Du and Dacheng Tao},
    booktitle={The Thirty-ninth Annual Conference on Neural Information Processing Systems},
    year={2025},
    url={https://openreview.net/forum?id=rdGMyTPhui}
}

@article{yang2025continual,
    title={Continual Model Merging without Data: Dual Projections for Balancing Stability and Plasticity},
    author={Enneng Yang and An Quang Tang and Li Shen and Gui-Bing Guo and Xingwei Wang and Xiaochun Cao and Jie Zhang},
    journal={Advances in Neural Information Processing Systems 38},
    year={2025},
    url={https://api.semanticscholar.org/CorpusID:290960643}
}

@article{Lanzillotta2023TowardsGF,
    title={Towards guarantees for parameter isolation in continual learning},
    author={Giulia Lanzillotta and Sidak Pal Singh and Benjamin F. Grewe and Thomas Hofmann},
    journal={ArXiv},
    year={2023},
    volume={abs/2310.01165},
    url={https://api.semanticscholar.org/CorpusID:263605982}
}

@inproceedings{Malladi2023,
    author={Sadhika Malladi and Alexander Wettig and Dingli Yu and Danqi Chen and Sanjeev Arora},
    title={A Kernel-Based View of Language Model Fine-Tuning},
    year={2023},
    cdate={1672531200000},
    pages={23610-23641},
    url={https://proceedings.mlr.press/v202/malladi23a.html},
    booktitle={ICML}
}

@inproceedings{
    zheng2026understanding,
    title={Understanding the Dynamics of Forgetting and Generalization in Continual Learning via the Neural Tangent Kernel},
    author={Guodong Zheng and Peng Wang and Shengchao Hu and Quan Zheng and Li Shen},
    booktitle={The Fourteenth International Conference on Learning Representations},
    year={2026},
    url={https://openreview.net/forum?id=NE2yIxdo1w}
}

@inproceedings{
    luo2026sphere,
    title={{SPHERE}: Mitigating the Loss of Spectral Plasticity in Mixture-of-Experts for Deep Reinforcement Learning},
    author={Lirui Luo and Guoxi Zhang and Hongming Xu and Cong Fang and Qing Li},
    booktitle={Forty-third International Conference on Machine Learning},
    year={2026},
    url={https://openreview.net/forum?id=hXyv6xeHkO}
}

@inproceedings{
    shenfeld2026rls,
    title={{RL}'s Razor: Why Online Reinforcement Learning Forgets Less},
    author={Idan Shenfeld and Jyothish Pari and Pulkit Agrawal},
    booktitle={The Fourteenth International Conference on Learning Representations},
    year={2026},
    url={https://openreview.net/forum?id=7HNRYT4V44}
}

@inproceedings{
    chen2026retaining,
    title={Retaining by Doing: The Role of On-Policy Data in Mitigating Forgetting},
    author={Howard Chen and Noam Razin and Karthik R Narasimhan and Danqi Chen},
    booktitle={Forty-third International Conference on Machine Learning},
    year={2026},
    url={https://openreview.net/forum?id=ODTM64azGa}
}

@inproceedings{
    bair2026compressed,
    title={Compressed Sensing for Capability Localization in Large Language Models},
    author={Anna Bair and Yixuan Even Xu and Mingjie Sun and J Zico Kolter},
    booktitle={Forty-third International Conference on Machine Learning},
    year={2026},
    url={https://openreview.net/forum?id=X4MFtxw1sB}
}

@inproceedings{michel2019,
    author={Paul Michel and Omer Levy and Graham Neubig},
    title={Are Sixteen Heads Really Better than One?},
    year={2019},
    cdate={1546300800000},
    pages={14014-14024},
    url={http://papers.nips.cc/paper/9551-are-sixteen-heads-really-better-than-one},
    booktitle={NeurIPS}
}

@inproceedings{voita2019analyzing,
    title = "Analyzing Multi-Head Self-Attention: Specialized Heads Do the Heavy Lifting, the Rest Can Be Pruned",
    author = "Voita, Elena  and
      Talbot, David  and
      Moiseev, Fedor  and
      Sennrich, Rico  and
      Titov, Ivan",
    booktitle = "Proceedings of the 57th Annual Meeting of the Association for Computational Linguistics",
    month = jul,
    year = "2019",
    address = "Florence, Italy",
    publisher = "Association for Computational Linguistics",
    url = "https://aclanthology.org/P19-1580/",
    doi = "10.18653/v1/P19-1580",
    pages = "5797--5808",
}

@inproceedings{
    shi2026understanding,
    title={Understanding the Performance Gap in Preference Learning: A Dichotomy of {RLHF} and {DPO}},
    author={Ruizhe Shi and Minhak Song and Runlong Zhou and Zihan Zhang and Maryam Fazel and Simon Shaolei Du},
    booktitle={Forty-third International Conference on Machine Learning},
    year={2026},
    url={https://openreview.net/forum?id=sg94PRd3kD}
}

@article{Lin2025ContinualLV,
    title={Continual Learning via Sparse Memory Finetuning},
    author={Jessy Lin and Luke S. Zettlemoyer and Gargi Ghosh and Wen-tau Yih and Aram H. Markosyan and Vincent-Pierre Berges and Barlas Ouguz},
    journal={ArXiv},
    year={2025},
    volume={abs/2510.15103},
    url={https://api.semanticscholar.org/CorpusID:282203348}
}

@misc{miles2026,
    title        = {Miles: Enterprise-Grade Reinforcement Learning for Large-Scale Model Post-Training},
    author       = {{Miles Team}},
    year         = {2026},
    howpublished = {\url{https://github.com/radixark/miles}}
}

@misc{qwen3technicalreport,
    title={Qwen3 Technical Report}, 
    author={{Qwen Team}},
    year={2025},
    eprint={2505.09388},
    archivePrefix={arXiv},
    primaryClass={cs.CL},
    url={https://arxiv.org/abs/2505.09388}, 
}

@inproceedings{
    yin2025which,
    title={Which Attention Heads Matter for In-Context Learning?},
    author={Kayo Yin and Jacob Steinhardt},
    booktitle={Forty-second International Conference on Machine Learning},
    year={2025},
    url={https://openreview.net/forum?id=C7XmEByCFv}
}

@article{Olsson2022IncontextLA,
    title={In-context Learning and Induction Heads},
    author={Catherine Olsson and Nelson Elhage and Neel Nanda and Nicholas Joseph and Nova Dassarma and Thomas Henighan and Benjamin Mann and Amanda Askell and Yuntao Bai and Anna Chen and Tom Conerly and Dawn Drain and Deep Ganguli and Zac Hatfield-Dodds and Danny Hernandez and Scott Johnston and Andy Jones and John Kernion and Liane Lovitt and Kamal Ndousse and Dario Amodei and Tom B. Brown and Jack Clark and Jared Kaplan and Sam McCandlish and Chris Olah},
    journal={ArXiv},
    year={2022},
    volume={abs/2209.11895},
    url={https://api.semanticscholar.org/CorpusID:252532078}
}

@inproceedings{
    yao2025leveraging,
    title={Leveraging Sparsity for Sample-Efficient Preference Learning: A Theoretical Perspective},
    author={Yunzhen Yao and Lie He and Michael Gastpar},
    booktitle={Forty-second International Conference on Machine Learning},
    year={2025},
    url={https://openreview.net/forum?id=nxCWKWok8N}
}

@inproceedings{guo-etal-2021-parameter,
    title = "Parameter-Efficient Transfer Learning with Diff Pruning",
    author = "Guo, Demi  and
      Rush, Alexander  and
      Kim, Yoon",
    editor = "Zong, Chengqing  and
      Xia, Fei  and
      Li, Wenjie  and
      Navigli, Roberto",
    booktitle = "Proceedings of the 59th Annual Meeting of the Association for Computational Linguistics and the 11th International Joint Conference on Natural Language Processing (Volume 1: Long Papers)",
    month = aug,
    year = "2021",
    address = "Online",
    publisher = "Association for Computational Linguistics",
    url = "https://aclanthology.org/2021.acl-long.378/",
    doi = "10.18653/v1/2021.acl-long.378",
    pages = "4884--4896"
}

@article{Duchi2011AdaptiveSM,
    title={Adaptive Subgradient Methods for Online Learning and Stochastic Optimization},
    author={John C. Duchi and Elad Hazan and Yoram Singer},
    journal={J. Mach. Learn. Res.},
    year={2011},
    volume={12},
    pages={2121-2159},
    url={https://api.semanticscholar.org/CorpusID:538820}
}

@inproceedings{lin2026parameter,
    title = "Parameter Importance is Not Static: Evolving Parameter Isolation for Supervised Fine-Tuning",
    author = "Lin, Zekai  and
      Xue, Chao  and
      Liang, Di  and
      Han, Xingsheng  and
      Liu, Peiyang  and
      Wu, Xianjie  and
      Jiang, Lei  and
      Lu, Yu  and
      Simons, Bob  and
      Liang, Shuang  and
      Peng, Minlong",
    editor = "Liakata, Maria  and
      Moreira, Viviane P.  and
      Zhang, Jiajun  and
      Jurgens, David",
    booktitle = "Proceedings of the 64th Annual Meeting of the {A}ssociation for {C}omputational {L}inguistics (Volume 1: Long Papers)",
    month = jul,
    year = "2026",
    address = "San Diego, California, United States",
    publisher = "Association for Computational Linguistics",
    url = "https://aclanthology.org/2026.acl-long.1516/",
    doi = "10.18653/v1/2026.acl-long.1516",
    pages = "32846--32864",
    ISBN = "979-8-89176-390-6"
}

@article{
    panos2025efficient,
    title={Efficient Few-Shot Continual Learning in Vision-Language Models},
    author={Aristeidis Panos and Rahaf Aljundi and Daniel Olmeda Reino and Richard E. Turner},
    journal={Transactions on Machine Learning Research},
    issn={2835-8856},
    year={2025},
    url={https://openreview.net/forum?id=sQ1w92WW0V},
    note={}
}
\bibliographystyle{misc/iclr2027_conference}

\newpage
\appendix
\section{Omitted Proofs}
\subsection{Omitted proofs in \Cref{sec:ranking-theory}}\label{sec:omitted-ranking-theory}
\begin{proof}[Proof of \Cref{cor:gradnorm-initial}]
By \Cref{prop:gradnorm_flow_eq},
\begin{align*}
    -\left.\frac{d}{dt}\mathcal L(\theta(t;S))\right|_{t=0}=\mathcal G(S)^2=\sum_{h\in S}\mathcal G(\{h\})^2\;.
\end{align*}
Hence selecting $k$ heads of largest gradient norm maximizes the initial loss decrease.
\end{proof}
\begin{proof}[Proof of \Cref{prop:gradnorm-ranking}]
Integrating \Cref{prop:gradnorm_flow_eq} gives
\begin{align*}
    \mathcal L(\theta_0)-\mathcal L(\theta(t;S))=\int_0^t\mathcal G(S;\theta(s;S))^2ds=t\,\overline{\mathcal G}(S;t)^2\;.
\end{align*}
Then for any $S$ with $\mathcal G(S)>0$ we have
\begin{align*}
    \mathcal L(\theta_0)-\mathcal L(\theta(t;S_k))=t\,\overline{\mathcal G}(S_k;t)^2\ge\frac{\overline{\mathcal G}(S_k;t)^2}{\overline{\mathcal G}(S;t)^2}\left[\mathcal L(\theta_0)-\mathcal L(\theta(t;S))\right]\;.
\end{align*}
Taking the maximum over $|S|\le k$ proves the claim.
\end{proof}
\begin{proof}[Proof of \Cref{prop:gradnorm_flow}]
Along the restricted gradient flow,
\begin{align*}
    \frac{d}{dt}\mathcal G(S;\theta(t;S))^2=-2(\nabla_{\theta_S}\mathcal L(\theta(t;S)))^\top\nabla_{\theta_S}^2\mathcal L(\theta(t;S))\nabla_{\theta_S}\mathcal L(\theta(t;S))\;.
\end{align*}
Hence, whenever $\mathcal G(S;\theta(t;S))>0$,
\begin{align*}
    -2\beta\mathcal G(S;\theta(t;S))^2\le\frac{d}{dt}\mathcal G(S;\theta(t;S))^2\le2\beta\mathcal G(S;\theta(t;S))^2\;.
\end{align*}
Integrating gives
\begin{align*}
    e^{-2\beta s}\mathcal G(S)^2\le\mathcal G(S;\theta(s;S))^2\le e^{2\beta s}\mathcal G(S)^2
\end{align*}
for every $s\in[0,t]$. Averaging over $s\in[0,t]$ yields
\[
    \frac{1-e^{-2\beta t}}{2\beta t}\mathcal G(S)^2\le\overline{\mathcal G}(S;t)^2\le\frac{e^{2\beta t}-1}{2\beta t}\mathcal G(S)^2\;.\qedhere
\]
\end{proof}
\begin{proof}[Proof of \Cref{prop:sparse-adam}]
The objective $\mathcal L_{\adam_\lambda}$ is strictly convex and separable across heads.
For head $h$, by the KKT condition, $\Delta_h=0$ is optimal iff
\begin{align*}
    0\in m_h+\lambda\partial\|\Delta_h\|_2\big|_{\Delta_h=0}=\{m_h+\lambda u:\|u\|_2\le1\}\;.
\end{align*}
This condition is equivalent to $\|m_h\|_2\le\lambda$. Strict convexity gives uniqueness.
\end{proof}
% \begin{proof}[Proof of \Cref{prop:sparse-adam-optimal}]
% By restricting $\Delta$ to $S$ and solving $\mathcal L_{\adam}(\Delta)$, we have $\Delta_{S}^\star=-\eta D_S^{-1}m$ and $\mathcal L_{\adam}(\Delta_{S}^\star)=-\frac{\eta m_S^\top D_{SS}^{-1}m_S}{2}$. Then we have that, for any $|S|\le k$,
% \begin{align*}
%     -\mathcal L_{\adam}(\Delta_{S_k^{(m)}}^\star)
%     &\ge \frac{\eta\|m_{S_k^{(m)}}\|_2^2}{2\max_{i\in\theta_{S_k^{(m)}}}(\sqrt{v_i}+\eps_{\adam})}
%     \ge \frac{\min_{i\in\theta_S}(\sqrt{v_i}+\eps_{\adam})}{\max_{i\in\theta_{S_k^{(m)}}}(\sqrt{v_i}+\eps_{\adam})}
%     \left[-\mathcal L_{\adam}(\Delta_S^\star)\right]\;,
% \end{align*}
% where the second inequality uses $\|m_{S_k^{(m)}}\|_2^2\ge\|m_S\|_2^2$.
% \end{proof}
\begin{proposition}
\label{prop:sparse-adam-optimal}
\rebut{Let $D_h=\operatorname{diag}(\sqrt{v_i}+\epsilon_{\adam})_{i\in\theta_h}$, $\Delta_{\mathrm{full}}^\star=-\eta D^{-1}m$ be the unconstrained minimizer of $\mathcal L_{\adam}$. Let $S_k^{(m)}$ contain the $k$ heads with largest $\|m_h\|_2$, and let $S_k^\star$ contain the $k$ heads with largest $m_h^\top D_h^{-1}m_h$. If the mismatched heads have nonzero first moments, then
\begin{align*}
\mathcal L_{\adam}(\Delta_{S_k^{(m)}}^\star)-\mathcal L_{\adam}(\Delta_{\mathrm{full}}^\star)\le \frac{\max_{h\in S_k^\star\setminus S_k^{(m)}}\frac{m_h^\top D_h^{-1}m_h}{\|m_h\|_2^2}}{\min_{h\in S_k^{(m)}\setminus S_k^\star}\frac{m_h^\top D_h^{-1}m_h}{\|m_h\|_2^2}}\min_{|S|\le k}
\left[\mathcal L_{\adam}(\Delta_S^\star)-\mathcal L_{\adam}(\Delta_{\mathrm{full}}^\star)\right]\;.
\end{align*}}
\end{proposition}

\begin{proof}[Proof of \Cref{prop:sparse-adam-optimal}]
\rebut{The $S_k^{(m)}=S_k^\star$ case is trivial, so we consider $S_k^{(m)}\ne S_k^\star$. Since $S_k^{(m)}$ contains the $k$ heads with largest $\|m_h\|_2$, we have $\sum_{h\in S_k^\star\setminus S_k^{(m)}}\|m_h\|_2^2\le\sum_{h\in S_k^{(m)}\setminus S_k^\star}\|m_h\|_2^2$. Therefore,
\begin{align*}
\sum_{h\in S_k^\star\setminus S_k^{(m)}}m_h^\top D_h^{-1}m_h\le \frac{\displaystyle\max_{h\in S_k^\star\setminus S_k^{(m)}}\frac{m_h^\top D_h^{-1}m_h}{\|m_h\|_2^2}}{\displaystyle\min_{h\in S_k^{(m)}\setminus S_k^\star}\frac{m_h^\top D_h^{-1}m_h}{\|m_h\|_2^2}}\sum_{h\in S_k^{(m)}\setminus S_k^\star}m_h^\top D_h^{-1}m_h.
\end{align*}
By the optimality of the restricted quadratic model,
\begin{align*}
&\quad\mathcal L_{\adam}(\Delta_{S_k^{(m)}}^\star)-\mathcal L_{\adam}(\Delta_{\mathrm{full}}^\star)\\
&=\frac{\eta}{2}\sum_{h\notin S_k^{(m)}}m_h^\top D_h^{-1}m_h\\
&=\frac{\eta}{2}\left(\sum_{h\notin S_k^\star}m_h^\top D_h^{-1}m_h+\sum_{h\in S_k^\star\setminus S_k^{(m)}}m_h^\top D_h^{-1}m_h-\sum_{h\in S_k^{(m)}\setminus S_k^\star}m_h^\top D_h^{-1}m_h\right)\\
&\le \frac{\eta}{2}\frac{\displaystyle\max_{h\in S_k^\star\setminus S_k^{(m)}}\frac{m_h^\top D_h^{-1}m_h}{\|m_h\|_2^2}}{\displaystyle\min_{h\in S_k^{(m)}\setminus S_k^\star}\frac{m_h^\top D_h^{-1}m_h}{\|m_h\|_2^2}}\sum_{h\notin S_k^\star}m_h^\top D_h^{-1}m_h\\
&=\frac{\displaystyle\max_{h\in S_k^\star\setminus S_k^{(m)}}\frac{m_h^\top D_h^{-1}m_h}{\|m_h\|_2^2}}{\displaystyle\min_{h\in S_k^{(m)}\setminus S_k^\star}\frac{m_h^\top D_h^{-1}m_h}{\|m_h\|_2^2}}\min_{|S|\le k}\left[\mathcal L_{\adam}(\Delta_S^\star)-\mathcal L_{\adam}(\Delta_{\mathrm{full}}^\star)\right],
\end{align*}
where the inequality follows from the preceding bound and the fact that the ratio is at least one.}
\end{proof}
\begin{remark}
\rebut{For any $h\in S_k^\star\setminus S_k^{(m)}$ and $h'\in S_k^{(m)}\setminus S_k^\star$, we have $\|m_h\|_2\le\|m_{h'}\|_2$ but $m_h^\top D_h^{-1}m_h\ge m_{h'}^\top D_{h'}^{-1}m_{h'}$. Intuitively, we expect the $\ell_2$ and Adam-preconditioned norms to be positively correlated across heads, so substantial ranking discrepancies should be relatively uncommon. Moreover, if the effective preconditioning factors of the mismatched heads remain comparable, the approximation factor can be close to one.}
\end{remark}
\subsection{Omitted proofs in \Cref{sec:partition-theory}}\label{sec:omitted-partition-theory}
\begin{proof}[Proof of \Cref{cor:gradnorm-forgetting-average}]
By \Cref{prop:continual-gradient-flow},
\begin{align*}
    \left|\mathcal L_A(\theta(t))-\mathcal L_A(\theta_A)\right|\le\int_0^t\mathcal G_A(S_B;\theta(s))\mathcal G_B(S_B;\theta(s))\,ds=\widetilde{\mathcal G}_{A\mid B}(S_B;t)\int_0^t\mathcal G_B(S_B;\theta(s))\,ds\;.
\end{align*}
On the other hand,
\[
    \mathcal L_B(\theta_A)-\mathcal L_B(\theta(t))=\int_0^t\mathcal G_B(S_B;\theta(s))^2\,ds=\widetilde{\mathcal G}_B(S_B;t)\int_0^t\mathcal G_B(S_B;\theta(s))\,ds\;.\qedhere
\]
\end{proof}

\begin{proof}[Proof of \Cref{prop:gradnorm-forgetting}]
By Cauchy--Schwarz inequality,
\begin{align*}
    \widetilde{\mathcal G}_B(S_B;t)=\frac{\int_0^t\mathcal G_B(S_B;\theta(s))^2\,ds}{\int_0^t\mathcal G_B(S_B;\theta(s))\,ds}\ge\sqrt{\frac{1}{t}\int_0^t\mathcal G_B(S_B;\theta(s))^2\,ds}\;.
\end{align*}
Then \Cref{prop:gradnorm_flow} gives
\begin{align*}
    \mathcal G_B(S_B;\theta(s))\ge\sqrt{\frac{1-e^{-2\beta_Bt}}{2\beta_Bt}}\,\mathcal G_B(S_B;\theta_A)\;.
\end{align*}
Note that
\begin{align*}
    \frac{d}{dt}\mathcal G_A(S_B;\theta(t))=-\frac{\nabla_{\theta_{S_B}}\mathcal L_A(\theta(t))^\top\nabla_{\theta_{S_B}}^2\mathcal L_A(\theta(t))\nabla_{\theta_{S_B}}\mathcal L_B(\theta(t))}{\|\nabla_{\theta_{S_B}}\mathcal L_A(\theta(t))\|_2}\le\beta_A\mathcal G_B(S_B;\theta(t))\;.
\end{align*}
Hence, for every $s\in[0,t]$,
\begin{align*}
    \mathcal G_A(S_B;\theta(s))\le\mathcal G_A(S_B;\theta_A)+\beta_A\int_0^s\mathcal G_B(S_B;\theta(u))\,du\;.
\end{align*}
Then we have
\begin{align*}
    \left|\mathcal L_A(\theta(t))-\mathcal L_A(\theta_A)\right|&\le\int_0^t\mathcal G_A(S_B;\theta(s))\mathcal G_B(S_B;\theta(s))\,ds\\
    &\le \int_0^t\left(\mathcal G_A(S_B;\theta_A)+\beta_A\int_0^s\mathcal G_B(S_B;\theta(u))\,du\right)\mathcal G_B(S_B;\theta(s))\,ds\\
    &=\mathcal G_A(S_B;\theta_A)\int_0^t\mathcal G_B(S_B;\theta(s))\,ds+\frac{\beta_A}{2}\left(\int_0^t\mathcal G_B(S_B;\theta(s))\,ds\right)^2\;.
\end{align*}
Using the fact that $\int_0^t\mathcal G_B(S_B;\theta(s))\,ds=\frac{\mathcal L_B(\theta_A)-\mathcal L_B(\theta(t))}{\widetilde{\mathcal G}_B(S_B;t)}$ gives the desired bound.
\end{proof}
\subsection{Omitted proofs in \Cref{sec:attention-curvature}}
\label{sec:proof-of-attention-curvature}

For an attention row $A_{i,:}$, define its attention-weighted mean and centered moments by
\begin{align*}
\mu_i:=\sum_jA_{ij}X_{j,:},\quad
M_{2,i}:=\sum_jA_{ij}(X_{j,:}-\mu_i)(X_{j,:}-\mu_i)^\top,\quad
M_{3,i}:=\sum_jA_{ij}(X_{j,:}-\mu_i)^{\otimes3}\;,
\end{align*}
where $x^{\otimes 3}$ denotes the third-order tensor product.

\begin{lemma}
\label{lem:attention-concentration-moments}
Suppose $\|X_{j,:}\|_2\le B$ for all $j$. If there exists $T_i$ s.t.
\begin{align*}
\sum_{j\in T_i}A_{ij}\ge1-\varepsilon,\qquad
\max_{j,j'\in T_i}\|X_{j,:}-X_{j',:}\|_2\le\rho\;,
\end{align*}
then
\begin{align*}
\|M_{2,i}\|_{\mathrm{op}}=O_B(\rho^2+\varepsilon),\qquad
\|M_{3,i}\|=O_B(\rho^3+\varepsilon)\;.
\end{align*}
\end{lemma}

\begin{proof}
Fix any $j_0\in T_i$. Since $\sum_jA_{ij}=1$,
\begin{align*}
\|\mu_i-X_{j_0,:}\|_2
&\le\sum_jA_{ij}\|X_{j,:}-X_{j_0,:}\|_2\\
&\le \rho+2B\varepsilon\;.
\end{align*}
Hence, for $j\in T_i$,
\begin{align*}
\|X_{j,:}-\mu_i\|_2\le 2\rho+2B\varepsilon\;,
\end{align*}
while for $j\notin T_i$,
\begin{align*}
\|X_{j,:}-\mu_i\|_2\le 2B\;.
\end{align*}
Therefore,
\begin{align*}
\|M_{2,i}\|_{\mathrm{op}}
&\le\sum_jA_{ij}\|X_{j,:}-\mu_i\|_2^2\\
&\le (2\rho+2B\varepsilon)^2+4B^2\varepsilon
=O_B(\rho^2+\varepsilon)\;,
\end{align*}
and similarly,
\begin{align*}
\|M_{3,i}\|
&\le\sum_jA_{ij}\|X_{j,:}-\mu_i\|_2^3\\
&\le (2\rho+2B\varepsilon)^3+8B^3\varepsilon
=O_B(\rho^3+\varepsilon)\;.
\end{align*}
\end{proof}

\begin{lemma}
\label{lem:attention-functional-hessian-moments}
Consider a single self-attention head
\begin{align*}
F_h=A_hXW_h^V,\qquad
A_h=\operatorname{softmax}\left(\frac{XW_h^Q(W_h^K)^\top X^\top}{\sqrt{d_k}}\right)\;,
\end{align*}
with parameter block $\theta_h=(W_h^Q,W_h^K,W_h^V)$. The second-derivative formulas of~\citet{ormaniec2024} express every nonzero block of the corresponding functional Hessian through the second and third centered attention moments. In particular, under bounded input and weight norms, there exists a constant $C$ depending only on these bounds and the dimensions s.t.
\begin{align*}
\|(\mathbf H_f)_{hh}\|_{\mathrm{op}}
\le
C\|\nabla_F\ell(F)\|_2
\left(
\max_i\|M_{2,i}\|_{\mathrm{op}}
+
\max_i\|M_{3,i}\|
\right)\;.
\end{align*}
\end{lemma}

\begin{proof}
For each head $h$, we further decompose its parameter block as
\begin{align*}
\theta_h=\left(W_h^Q,W_h^K,W_h^V\right)\;.
\end{align*}
Accordingly, the within-head functional-Hessian block can be written as
\begin{align*}
(\mathbf H_f)_{hh}
=
\begin{pmatrix}
(\mathbf H_f)_{W_h^QW_h^Q} & (\mathbf H_f)_{W_h^QW_h^K} & (\mathbf H_f)_{W_h^QW_h^V}\\
(\mathbf H_f)_{W_h^KW_h^Q} & (\mathbf H_f)_{W_h^KW_h^K} & (\mathbf H_f)_{W_h^KW_h^V}\\
(\mathbf H_f)_{W_h^VW_h^Q} & (\mathbf H_f)_{W_h^VW_h^K} & (\mathbf H_f)_{W_h^VW_h^V}
\end{pmatrix}\;.
\end{align*}
By Theorem~3.2 and Remark~3.1 of~\citet{ormaniec2024}, each nonzero $(W_h^Q,W_h^K,W_h^V)$ block above can be written as a linear function of $M_2$, $M_3$, or both. Under bounded input and weight norms, this implies
\[
\|(\mathbf H_f)_{hh}\|_{\mathrm{op}}
\le
C\|\nabla_F\ell(F)\|_2
\left(
\max_i\|M_{2,i}\|_{\mathrm{op}}
+
\max_i\|M_{3,i}\|
\right)\;.\qedhere
\]
\end{proof}
\begin{proof}[Proof of \Cref{prop:head-local-functional-curvature}]
The output of a multi-head self-attention layer is the sum of head outputs, each depending only on its own parameter block. Hence, for $h\neq h'$,
\begin{align*}
\nabla_{\theta_h,\theta_{h'}}^2F=0\;.
\end{align*}
The definition of the functional Hessian therefore gives $(\mathbf H_f)_{hh'}=0$. Thus $(\mathbf H_f)_{SS}$ is block diagonal, whose operator norm is the maximum operator norm of its diagonal blocks.
\end{proof}
\begin{proof}[Proof of \Cref{prop:concentrated-attention-curvature}]
By \Cref{lem:attention-concentration-moments},
\begin{align*}
\max_i\|M_{2,i}\|_{\mathrm{op}}=O(\rho^2+\varepsilon),\qquad
\max_i\|M_{3,i}\|=O(\rho^3+\varepsilon)\;.
\end{align*}
Applying \Cref{lem:attention-functional-hessian-moments} and using $\|\nabla_F\ell(F)\|_2\le G$ yields
\[
\|(\mathbf H_f)_{hh}\|_{\mathrm{op}}
=
O\left(G\left(\rho^2+\rho^3+\varepsilon\right)\right)\;.\qedhere
\]
\end{proof}
\begin{proof}[Proof of \Cref{prop:outer-hessian-entk}]
We have
\[
\|(\mathbf H_o)_{SS}\|_{\mathrm{op}}
=\left\|\left(\frac{\partial F}{\partial\theta_S}\right)^\top\nabla_F^2\ell(F)\left(\frac{\partial F}{\partial\theta_S}\right)\right\|_{\mathrm{op}}\le \left\|\frac{\partial F}{\partial\theta_S}\right\|_{\mathrm{op}}^2\|\nabla_F^2\ell(F)\|_{\mathrm{op}}\le \beta\|K_S\|_{\mathrm{op}}\;.\qedhere
\]
\end{proof}
\begin{proof}[Proof of \Cref{cor:attention-second-order}]
Note that
\begin{align*}
    \left|(\nabla_{\theta_S}\mathcal L)^\top\mathbf H_{SS}v_S\right|\le \|\mathbf H_{SS}\|_{\mathrm{op}}\|\nabla_{\theta_S}\mathcal L\|_2\|v\|_2\le \left(\|\mathbf (\mathbf H_f)_{SS}\|_{\mathrm{op}}+\|(\mathbf H_o)_{SS}\|_{\mathrm{op}}\right)\|\nabla_{\theta_S}\mathcal L\|_2\|v\|_2\;.
\end{align*}
Combining this with \Cref{prop:head-local-functional-curvature,prop:outer-hessian-entk} gives the desired bound.
\end{proof}
\newpage
\section{Supplemental Experimental Results}
\subsection{Supplemental training curves}\label{sec:supp_curves}
\begin{figure}[htbp]
\centering
\includegraphics[width=0.98\textwidth]{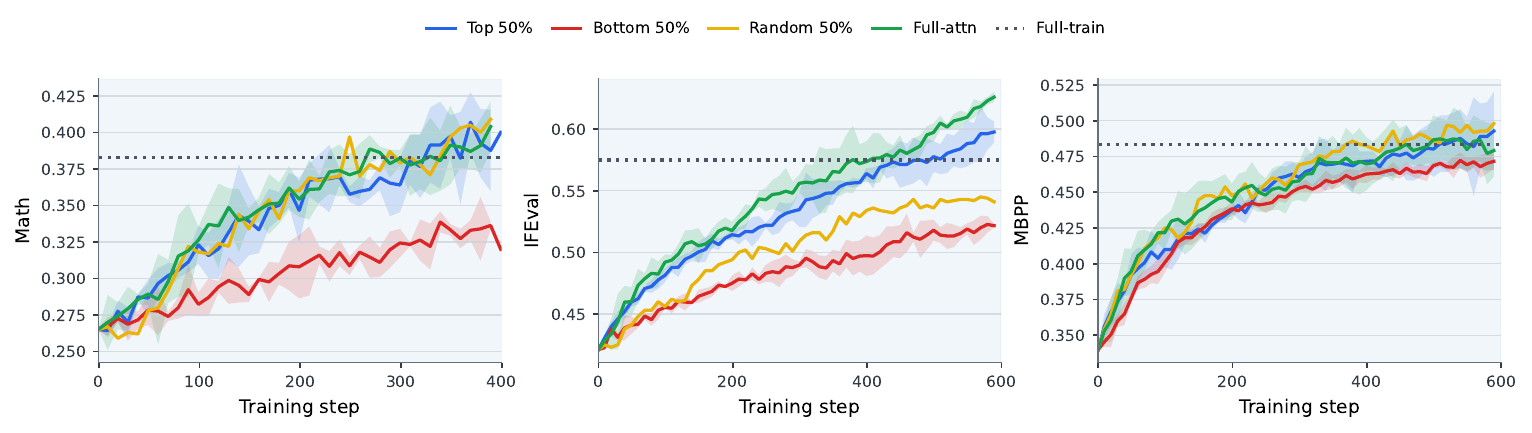}
\caption{Sparse RL training accuracy curves of \lm{Qwen3-1.7B}. Methods include top 50\%, \rebut{bottom 50\%, and random 50\%} attention-head tuning, full-attention tuning, and full tuning. Full tuning is shown as a final-score horizontal line.}
\label{fig:sparse-curves-qwen317b}
\end{figure}

\begin{figure}[hbtp]
\centering
\includegraphics[width=0.98\textwidth]{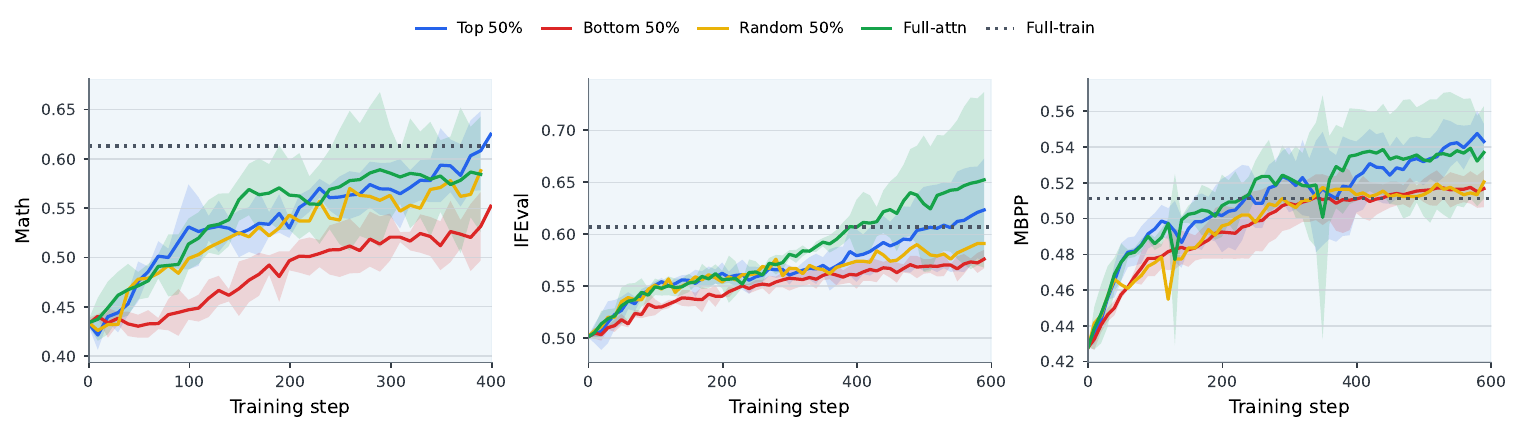}
\caption{Sparse RL training accuracy curves of \lm{Qwen3-4B}. Methods include top 50\%, \rebut{bottom 50\%, and random 50\%} attention-head tuning, full-attention tuning, and full tuning. Full tuning is shown as a final-score horizontal line.}
\label{fig:sparse-curves-qwen34b}
\end{figure}

\begin{figure}[htbp]
\centering
\includegraphics[width=0.49\textwidth]{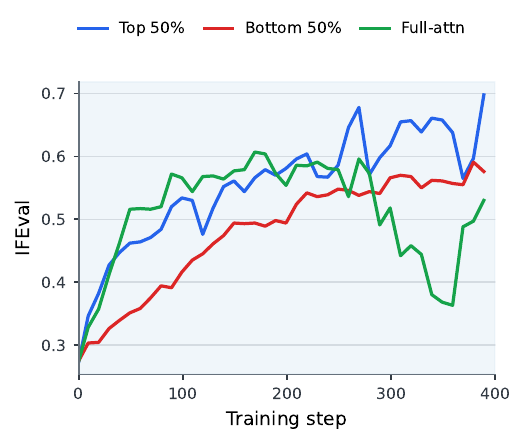}
\caption{Sparse RL training accuracy curves of \lm{Mistral-7B}. Methods include top 50\% and bottom 50\% attention-head tuning, and full-attention tuning. }
\label{fig:sparse-curves-mistral7b}
\end{figure}

\subsection{Supplemental numerical results}\label{sec:supp_tables}
\begin{table}[H]
\caption{Continual RL accuracy (\%) of \lm{Qwen3-1.7B} (Math $\rightarrow$ Coding). The model is first trained on Math task and then on Coding task. Parenthesized arrows are performance changes from the previous stage. The highest performance across all methods is highlighted in pink. The performance before training on the corresponding task is highlighted in gray.}
\label{tab:continual-qwen317bmathmbpp}
\centering
\setlength{\tabcolsep}{5pt}
\begin{tabular}{lllcc}
\toprule
Method & Train stage & Eval split & Math metric & Coding metric \\
\midrule
Top 50\% $\rightarrow$ bottom 50\% & $\mathrm{I}$ & In-dist test & \cellcolor[RGB]{255,225,235} 37.5 & \cellcolor[gray]{0.88} 56.6  \\
 &  & Out-dist test & \cellcolor[RGB]{255,225,235} 75.6 & \cellcolor[gray]{0.88} 37.0  \\
\addlinespace[1pt]
 & $\mathrm{II}$ & In-dist test & \cellcolor[RGB]{255,225,235} 35.8 ($\downarrow 1.7$) &  56.1  \\
 &  & Out-dist test & \cellcolor[RGB]{255,225,235} 75.2 ($\downarrow 0.4$) & \cellcolor[RGB]{255,225,235} 37.7  \\
\addlinespace[1pt]
Full-attn continual & $\mathrm{I}$ & In-dist test & 33.1 & \cellcolor[gray]{0.88} 54.8  \\
 &  & Out-dist test & 74.2 & \cellcolor[gray]{0.88} 35.2  \\
\addlinespace[1pt]
 & $\mathrm{II}$ & In-dist test & 30.0 ($\downarrow 3.1$) & \cellcolor[RGB]{255,225,235} 60.1  \\
 &  & Out-dist test & 73.0 ($\downarrow 1.2$) & 36.9  \\
\addlinespace[1pt]
Full continual & $\mathrm{I}$ & In-dist test & 35.3 & \cellcolor[gray]{0.88} 55.0  \\
 &  & Out-dist test & 73.6 & \cellcolor[gray]{0.88} 35.0  \\
\addlinespace[1pt]
 & $\mathrm{II}$ & In-dist test & 34.5 ($\downarrow 0.8$) & 58.5  \\
 &  & Out-dist test & 73.8 ($\uparrow 0.2$) & 33.6  \\
\bottomrule
\end{tabular}
\end{table}

\begin{table}[H]
\caption{Continual RL accuracy (\%) of \lm{Qwen3-1.7B} (IF $\rightarrow$ Math). The model is first trained on IF task and then on Math task. Parenthesized arrows are performance changes from the previous stage. The highest performance across all methods is highlighted in pink. The performance before training on the corresponding task is highlighted in gray.}
\label{tab:continual-qwen317bifevalmath}
\centering
\setlength{\tabcolsep}{5pt}
\begin{tabular}{lllcc}
\toprule
Method & Train stage & Eval split & IF metric & Math metric \\
\midrule
Top 50\% $\rightarrow$ bottom 50\% & $\mathrm{I}$ & In-dist test & 58.0 & \cellcolor[gray]{0.88} 24.6  \\
 &  & Out-dist test & \cellcolor[RGB]{255,225,235} 20.3 & \cellcolor[gray]{0.88} 70.6  \\
\addlinespace[1pt]
 & $\mathrm{II}$ & In-dist test & 57.7 ($\downarrow 0.3$) & 30.5  \\
 &  & Out-dist test & 20.7 ($\uparrow 0.3$) & 73.8  \\
\addlinespace[1pt]
Round robin 50\%+50\% & $\mathrm{I}$ & In-dist test & 52.5 & \cellcolor[gray]{0.88} 27.2  \\
 &  & Out-dist test & 19.0 & \cellcolor[gray]{0.88} 70.0  \\
\addlinespace[1pt]
 & $\mathrm{II}$ & In-dist test & 52.2 ($\downarrow 0.3$) & 33.9  \\
 &  & Out-dist test & 18.7 ($\downarrow 0.3$) & 74.0  \\
\addlinespace[1pt]
Full-attn continual & $\mathrm{I}$ & In-dist test & \cellcolor[RGB]{255,225,235} 59.4 & \cellcolor[gray]{0.88} 24.0  \\
 &  & Out-dist test & 20.0 & \cellcolor[gray]{0.88} 65.0  \\
\addlinespace[1pt]
 & $\mathrm{II}$ & In-dist test & \cellcolor[RGB]{255,225,235} 59.3 ($\downarrow 0.1$) & 35.7  \\
 &  & Out-dist test & 19.3 ($\downarrow 0.7$) & \cellcolor[RGB]{255,225,235} 78.2  \\
\addlinespace[1pt]
Full continual & $\mathrm{I}$ & In-dist test & 54.6 & \cellcolor[gray]{0.88} 25.7  \\
 &  & Out-dist test & 19.7 & \cellcolor[gray]{0.88} 65.2  \\
\addlinespace[1pt]
 & $\mathrm{II}$ & In-dist test & 54.4 ($\downarrow 0.2$) & \cellcolor[RGB]{255,225,235} 36.1  \\
 &  & Out-dist test & \cellcolor[RGB]{255,225,235} 21.7 ($\uparrow 2.0$) & 76.8  \\
\bottomrule
\end{tabular}
\end{table}

\begin{table}[H]
\caption{Continual RL accuracy (\%) of \lm{Qwen3-4B} (Coding $\rightarrow$ IF). The model is first trained on Coding task and then on IF task. Parenthesized arrows are performance changes from the previous stage. The highest performance across all methods is highlighted in pink. The performance before training on the corresponding task is highlighted in gray.}
\label{tab:continual-qwen34bmbppifeval}
\centering
\setlength{\tabcolsep}{5pt}
\begin{tabular}{lllcc}
\toprule
Method & Train stage & Eval split & Coding metric & IF metric \\
\midrule
Top 50\% $\rightarrow$ bottom 50\% & $\mathrm{I}$ & In-dist test &\cellcolor[RGB]{255,225,235} 69.0 & \cellcolor[gray]{0.88} 49.9 \\
 &  & Out-dist test & \cellcolor[RGB]{255,225,235} 51.4 & \cellcolor[gray]{0.88} 24.3  \\
\addlinespace[1pt]
 & $\mathrm{II}$ & In-dist test & \cellcolor[RGB]{255,225,235} 68.3 ($\downarrow 0.7$) & 53.1  \\
 &  & Out-dist test & \cellcolor[RGB]{255,225,235} 52.6 ($\uparrow 1.2$) & 23.0  \\
\addlinespace[1pt]
Round robin 50\%+50\% & $\mathrm{I}$ & In-dist test & 68.0 & \cellcolor[gray]{0.88} 49.8 \\
 &  & Out-dist test & 42.9 & \cellcolor[gray]{0.88} 25.7  \\
\addlinespace[1pt]
 & $\mathrm{II}$ & In-dist test & 64.3 ($\downarrow 3.7$) & 55.5  \\
 &  & Out-dist test & 34.6 ($\downarrow 8.3$) & \cellcolor[RGB]{255,225,235} 25.0  \\
\addlinespace[1pt]
Full-attn continual & $\mathrm{I}$ & In-dist test & 65.1 & \cellcolor[gray]{0.88} 51.8 \\
 &  & Out-dist test & 47.9 & \cellcolor[gray]{0.88} 22.0  \\
\addlinespace[1pt]
 & $\mathrm{II}$ & In-dist test & 54.8 ($\downarrow 10.3$) & \cellcolor[RGB]{255,225,235} 62.2  \\
 &  & Out-dist test & 34.1 ($\downarrow 13.8$) & 22.3  \\
\addlinespace[1pt]
Full continual & $\mathrm{I}$ & In-dist test & 66.4 & \cellcolor[gray]{0.88} 52.1 \\
 &  & Out-dist test & 43.6 & \cellcolor[gray]{0.88} 22.7  \\
\addlinespace[1pt]
 & $\mathrm{II}$ & In-dist test & 65.9 ($\downarrow 0.5$) & 56.0  \\
 &  & Out-dist test & 41.7 ($\downarrow 1.9$) & \cellcolor[RGB]{255,225,235} 25.0  \\
\bottomrule
\end{tabular}
\end{table}

\begin{table}[H]
\caption{
Continual RL accuracy (\%) of \lm{Qwen3-4B} (Math $\rightarrow$ IF). The model is first trained on Math task and then on IF task. Parenthesized arrows are performance changes from the previous stage. The highest performance across all methods is highlighted in pink. The performance before training on the corresponding task is highlighted in gray. }
\label{tab:continual-qwen34bmathifeval}
\centering
\setlength{\tabcolsep}{5pt}
\begin{tabular}{lllcc}
\toprule
Method & Train stage & Eval split & Math metric & IF metric \\
\midrule
Top 50\% $\rightarrow$ bottom 50\% & $\mathrm{I}$ & In-dist test & \cellcolor[RGB]{255,225,235} 59.3 & \cellcolor[gray]{0.88} 49.6  \\
 &  & Out-dist test & \cellcolor[RGB]{255,225,235} 88.4 & \cellcolor[gray]{0.88} 22.7  \\
\addlinespace[1pt]
 & $\mathrm{II}$ & In-dist test & \cellcolor[RGB]{255,225,235} 60.0 ($\uparrow 0.7$) & 53.6  \\
 &  & Out-dist test & \cellcolor[RGB]{255,225,235} 87.2 ($\downarrow 1.2$) & \cellcolor[RGB]{255,225,235} 26.0  \\
\addlinespace[1pt]
Round robin 50\%+50\% & $\mathrm{I}$ & In-dist test & 53.2 & \cellcolor[gray]{0.88} 50.3  \\
 &  & Out-dist test & 86.0 & \cellcolor[gray]{0.88} 23.3  \\
\addlinespace[1pt]
 & $\mathrm{II}$ & In-dist test & 53.2 ($\downarrow 0.0$) & 56.4  \\
 &  & Out-dist test & 85.6 ($\downarrow 0.4$) & 25.3  \\
\addlinespace[1pt]
Full-attn continual & $\mathrm{I}$ & In-dist test & 55.3 & \cellcolor[gray]{0.88} 48.9  \\
 &  & Out-dist test & 86.0 & \cellcolor[gray]{0.88} 21.3  \\
\addlinespace[1pt]
 & $\mathrm{II}$ & In-dist test & 51.2 ($\downarrow 4.1$) & 58.8  \\
 &  & Out-dist test & 84.2 ($\downarrow 1.8$) & \cellcolor[RGB]{255,225,235} 26.0  \\
\addlinespace[1pt]
Full continual & $\mathrm{I}$ & In-dist test & 58.2 & \cellcolor[gray]{0.88} 49.8  \\
 &  & Out-dist test & 88.2 & \cellcolor[gray]{0.88} 24.0  \\
\addlinespace[1pt]
 & $\mathrm{II}$ & In-dist test & 57.2 ($\downarrow 1.0$) & \cellcolor[RGB]{255,225,235} 59.0  \\
 &  & Out-dist test & 86.8 ($\downarrow 1.4$) & 24.3  \\
\bottomrule
\end{tabular}
\end{table}

\begin{table}[H]
\caption{Continual RL accuracy (\%) of \lm{Qwen3-1.7B} (Math $\rightarrow$ Coding $\rightarrow$ IF). The model is first trained on Math task, then on Coding task, and finally on IF task. Parenthesized arrows are performance changes from the previous stage. The performance before training on the corresponding task is highlighted in gray.}
\label{tab:continual-qwen317bmathmbppifeval}
\centering
\setlength{\tabcolsep}{4pt}
\begin{tabular}{llccc}
\toprule
Train stage & Eval split & Math metric & Coding metric & IF metric \\
\midrule
$\mathrm{I}$ & In-dist test & 34.5 & \cellcolor[gray]{0.88} 55.6 &  \\
 & Out-dist test & 76.2 & \cellcolor[gray]{0.88} 34.9 &  \\
\addlinespace[1pt]
$\mathrm{II}$ & In-dist test & 32.4 ($\downarrow 2.1$) & 59.3 & \cellcolor[gray]{0.88} 41.7 \\
 & Out-dist test & 75.0 ($\downarrow 1.2$) & 35.2 & \cellcolor[gray]{0.88} 14.0 \\
\addlinespace[1pt]
$\mathrm{III}$ & In-dist test & 33.6 ($\uparrow 1.2$) & 58.7 ($\downarrow 0.6$) & 47.3 \\
 & Out-dist test & 74.8 ($\downarrow 0.2$) & 34.0 ($\downarrow 1.2$) & 17.7 \\
\bottomrule
\end{tabular}
\end{table}

\begin{table}[htbp]
\caption{Continual RL accuracy (\%) of \lm{Qwen3-4B} (Math $\rightarrow$ IF $\rightarrow$ Coding). The model is first trained on Math task, then on IF task, and finally on Coding task. Parenthesized arrows are performance changes from the previous stage. The performance before training on the corresponding task is highlighted in gray.}
\label{tab:continual-qwen34bmathifevalmbpp}
\centering
\setlength{\tabcolsep}{4pt}
\begin{tabular}{llccc}
\toprule
Train stage & Eval split & Math metric & IF metric & Coding metric \\
\midrule
$\mathrm{I}$ & In-dist test & 55.7 & \cellcolor[gray]{0.88} 50.2 &  \\
 & Out-dist test & 86.8 & \cellcolor[gray]{0.88} 21.3 &  \\
\addlinespace[1pt]
$\mathrm{II}$ & In-dist test & 53.6 ($\downarrow 2.1$) & 57.7 &\cellcolor[gray]{0.88} 66.4 \\
 & Out-dist test & 83.0 ($\downarrow 3.8$) & 24.0 & \cellcolor[gray]{0.88} 43.5 \\
\addlinespace[1pt]
$\mathrm{III}$ & In-dist test & 53.9 ($\uparrow 0.3$) & 56.7 ($\downarrow 1.0$) & 67.2 \\
 & Out-dist test & 85.0 ($\uparrow 2.0$) & 23.3 ($\downarrow 0.7$) & 49.2  \\
\bottomrule
\end{tabular}
\end{table}

\begin{table}[H]
\caption{Head-reset experiments after continual RL. After completing stage $\mathrm{II}$, we reset the heads updated during stage $\mathrm{I}$ to their base-model values. We report the resulting test accuracy (\%) changes. Stage-$\mathrm{II}$ performance that does not decrease after the reset is highlighted in orange. The results suggest that the adaptations learned on the disjoint head sets are functionally localized.}
\label{tab:stage1-head-reset}
\centering
\setlength{\tabcolsep}{5pt}
\begin{tabular}{lllcccc}
\toprule
& & & \multicolumn{2}{c}{Stage $\mathrm I$ task} & \multicolumn{2}{c}{Stage $\mathrm{II}$ task} \\
\cmidrule(lr){4-5}\cmidrule(lr){6-7}
Model & Task sequence & Head allocation & In-dist & Out-dist & In-dist & Out-dist \\
\midrule
\lm{Qwen3-1.7B} & Math $\rightarrow$ Coding & Partition & $-10.7$ & $-6.0$ & \cellcolor{orange!20} $+1.3$ & $-4.0$  \\
\midrule
\lm{Qwen3-1.7B} & IF $\rightarrow$ Math & Partition & $-16.4$ & $-4.3$ & \cellcolor{orange!20}$+1.1$ & \cellcolor{orange!20}$+1.8$ \\
& & Round robin & $-11.6$ & $-1.67$ & $-3.3$ & $-0.4$\\
\midrule
\lm{Qwen3-4B} & Coding $\rightarrow$ IF & Partition & $-4.0$ & $-6.4$ & \cellcolor{orange!20}$+0.4$ & \cellcolor{orange!20}$+0.0$ \\
& & Round robin & $-0.8$  & $+6.0$ & \cellcolor{orange!20}$+0.6$ & \cellcolor{orange!20}$+2.3$ \\
\midrule
\lm{Qwen3-4B} & Math $\rightarrow$ IF & Partition & $-17.4$ & $-7.2$ & \cellcolor{orange!20}$+0.8$ & $-2.3$ \\
& & Round robin & $-9.9$ & $-5.8$ & \cellcolor{orange!20}$+0.8$ & \cellcolor{orange!20}$+0.3$ \\
\bottomrule
\end{tabular}
\end{table}
\newpage
\subsection{\rebut{Supplemental empirical analysis}}\label{sec:other_supp_tables}
\begin{table}[H]
\centering
\small
\caption{Trainable parameters and FP32-equivalent gradient memory (4 bytes per parameter).}
\label{tab:efficiency}
\begin{tabular}{lrrr}
\toprule
Model & 50\% Attn & Full Attn & Full \\
\midrule
\lm{Qwen3-1.7B} & 0.176B / 0.656 GiB & 0.352B / 1.312 GiB & 1.721B / 6.410 GiB \\
\lm{Qwen3-4B} & 0.472B / 1.758 GiB & 0.944B / 3.516 GiB & 4.022B / 14.985 GiB \\
\lm{Mistral-7B-Instruct-v0.3} & 0.671B / 2.500 GiB & 1.342B / 5.000 GiB & 7.248B / 27.001 GiB \\
\bottomrule
\end{tabular}
\end{table}
\begin{table}[H]
\centering
\caption{
Spearman distances ($d:=1-\text{Spearman correlation}$) between head rankings, where lower values indicate more similar rankings, identical rankings have distance $0$, and two independent uniformly random rankings have expected distance $1$. For Math and IF, within-task distances are averaged over the three pairwise comparisons among disjoint 1k-subsets of \texttt{DAPO-Math-17k} and \texttt{RLVR-IFEval}. For Coding, the within-task distance is computed between two disjoint 487-subsets of \texttt{MBPP}. Cross-task distances are computed using the training sets.}
\label{tab:head-order-rank-distance}
\begin{tabular}{lcccccc}
\toprule
& \multicolumn{3}{c}{Within-task}
& \multicolumn{3}{c}{Cross-task} \\
\cmidrule(lr){2-4}
\cmidrule(lr){5-7}
Model
& Math
& IF
& Coding
& Math \& IF
& Math \& Coding
& IF \& Coding \\
\midrule
\lm{Qwen3-1.7B}
& 0.025 & 0.026 & 0.051
& 0.044 & 0.084 & 0.081 \\
\lm{Qwen3-4B}
& 0.067 & 0.040 & 0.047
& 0.076 & 0.121 & 0.104 \\
\bottomrule
\end{tabular}
\end{table}

\begin{table}[H]
\centering
\caption{\rebut{Consistency of top-50\% attention-head rankings throughout sparse RL fine-tuning. We report the overlap ratio (\%) between the top-50\% heads identified at each training step and those at initialization, across different models and tasks. The initially selected top-50\% heads generally maintain their gradient-norm advantage throughout training.}}
\label{tab:ranking-consistency}
\rebut{\begin{tabular}{lcccc}
\toprule
Model & Step & Math & IF & Coding \\
\midrule
\lm{Qwen3-1.7B} & 0 & 100.0 & 100.0 & 100.0 \\
& 200 & 92.9 & 86.6 & 88.4 \\
& 400 & 94.6 & 72.3 & 87.5 \\
& 600 &  & 91.1 & 88.4 \\
\midrule
\lm{Qwen3-4B} & 0 & 100.0 & 100.0 & 100.0 \\
& 200 & 88.2 & 93.8 & 91.0 \\
& 400 & 84.7 & 93.1 & 92.4 \\
& 600 &  & 91.7 & 91.7 \\
\bottomrule
\end{tabular}}
\end{table}
\begin{table}[H]
\centering
\caption{\rebut{Directional second-order quantities measured at different training steps throughout sparse RL fine-tuning. We report $\frac{\left|(\nabla_{\theta_S}\mathcal L(\theta(t;S)))^\top\nabla_{\theta_S}^2\mathcal L(\theta(t;S))\nabla_{\theta_S}\mathcal L(\theta(t;S))\right|}{\Vert\nabla_{\theta_S}\mathcal L(\theta(t;S))\Vert_2^2}$.}}
\label{tab:sparse-second-order}
\rebut{\begin{tabular}{lcccc}
\toprule
Model & Step & Math & IF & Coding \\
\midrule
\lm{Qwen3-1.7B} & 0
& $7.97\times 10^{-2}$ & $2.35\times 10^{3}$ & $3.22\times 10^{1}$ \\
& 200
& $1.02\times 10^{-2}$ & $8.18\times 10^{0}$ & $7.11\times 10^{0}$ \\
& 400
& $9.16\times 10^{-2}$ & $1.48\times 10^{3}$ & $1.37\times 10^{1}$ \\
& 600
&  & $8.56\times 10^{1}$ & $4.80\times 10^{0}$ \\
\midrule
\lm{Qwen3-4B} & 0
& $1.45\times 10^{2}$ & $8.51\times 10^{1}$ & $1.09\times 10^{2}$ \\
& 200
& $5.13\times 10^{2}$ & $2.71\times 10^{2}$ & $1.62\times 10^{1}$ \\
& 400
& $4.53\times 10^{2}$ & $5.45\times 10^{3}$ & $4.10\times 10^{1}$ \\
& 600
&  & $1.55\times 10^{3}$ & $2.57\times 10^{1}$ \\
\bottomrule
\end{tabular}}
\end{table}

\begin{table}[H]
\centering
\caption{\rebut{Directional second-order quantities measured at different training steps throughout the stage-II of continual RL fine-tuning on tasks $A\rightarrow B$, using the top-50\% $\rightarrow$ bottom-50\% head allocation strategy. The old-task and new-task terms are
$\frac{\vert \nabla_{\theta_{S_B}}\mathcal L_A(\theta(t))^\top\nabla_{\theta_{S_B}}^2\mathcal L_A(\theta(t))\nabla_{\theta_{S_B}}\mathcal L_B(\theta(t))\vert}{\|\nabla_{\theta_{S_B}}\mathcal L_A(\theta(t))\|_2\|\nabla_{\theta_{S_B}}\mathcal L_B(\theta(t))\|_2}$ and
$\frac{\vert \nabla_{\theta_{S_B}}\mathcal L_B(\theta(t))^\top\nabla_{\theta_{S_B}}^2\mathcal L_B(\theta(t))\nabla_{\theta_{S_B}}\mathcal L_B(\theta(t))\vert}{\|\nabla_{\theta_{S_B}}\mathcal L_B(\theta(t))\|_2^2}$, respectively.
Step 0 denotes the stage-I final checkpoint before training task $B$.}}
\label{tab:continual-second-order}
\setlength{\tabcolsep}{6pt}
\rebut{
\begin{tabular}{llccc}
\toprule
Model & Task sequence & Stage-II step
& old task & new task\\
\midrule
\lm{Qwen3-1.7B} & Math $\rightarrow$ Coding & 0
& $5.90\times 10^{-2}$ & $1.97\times 10^{1}$ \\
& & 200
& $4.38\times 10^{-3}$ & $2.39\times 10^{-1}$ \\
& & 400
& $6.95\times 10^{-2}$ & $1.26\times 10^{0}$ \\
& & 600
& $4.98\times 10^{-2}$ & $3.39\times 10^{0}$ \\
\midrule
\lm{Qwen3-4B} & Math $\rightarrow$ IF & 0
& $8.99\times 10^{0}$ & $6.42\times 10^{1}$ \\
& & 200
& $2.74\times 10^{0}$ & $8.26\times 10^{1}$ \\
& & 400
& $8.24\times 10^{1}$ & $5.55\times 10^{0}$ \\
& & 600
& $1.49\times 10^{1}$ & $9.72\times 10^{1}$ \\
\addlinespace[3pt]
& Coding $\rightarrow$ IF & 0
& $9.34\times 10^{-2}$ & $1.87\times 10^{2}$ \\
& & 200
& $3.68\times 10^{-1}$ & $2.47\times 10^{1}$ \\
& & 400
& $5.85\times 10^{-1}$ & $6.73\times 10^{1}$ \\
& & 600
& $7.65\times 10^{-4}$ & $2.91\times 10^{1}$ \\
\bottomrule
\end{tabular}}
\end{table}

\begin{table}[H]
\centering
\caption{\rebut{Empirical approximation ratios for $\ell_2$ moment-based head selection on \lm{Qwen3-1.7B} and \lm{Qwen3-4B}. We report $\frac{\mathcal L_{\adam}(\Delta_{S_k^{(m)}}^\star)-\mathcal L_{\adam}(\Delta_{\mathrm{full}}^\star)}{\min_{|S|\le k}\left[\mathcal L_{\adam}(\Delta_S^\star)-\mathcal L_{\adam}(\Delta_{\mathrm{full}}^\star)\right]}$, where $S_k^{(m)}$ selects the top 50\% of heads by $\|m_h\|_2$, and $\Delta_{\mathrm{full}}^\star=-\eta D^{-1}m$ is the unconstrained minimizer of $\mathcal L_{\adam}$. All quantities are computed over Q/K/V parameters during full-attention fine-tuning. Ratios remain close to $1$, indicating that our selection achieves a similar quadratic loss gap to the optimal selection under the same head budget.}}
\label{tab:adam-head-bound}
\setlength{\tabcolsep}{10pt}
\rebut{
\begin{tabular}{ccccccc}
\toprule
& \multicolumn{3}{c}{\lm{Qwen3-1.7B}}
& \multicolumn{3}{c}{\lm{Qwen3-4B}} \\
\cmidrule(lr){2-4}\cmidrule(lr){5-7}
Step & Math & IF & Coding & Math & IF & Coding \\
\midrule
10  & 1.0452 & 1.0496 & 1.0930 & 1.0255 & 1.0230 & 1.0819 \\
20  & 1.0347 & 1.0387 & 1.0696 & 1.0367 & 1.0221 & 1.0489 \\
30  & 1.0372 & 1.0322 & 1.0550 & 1.0557 & 1.0315 & 1.0411 \\
40  & 1.0290 & 1.0474 & 1.1595 & 1.0359 & 1.0338 & 1.0466 \\
50  & 1.0222 & 1.0329 & 1.0627 & 1.0302 & 1.0282 & 1.0614 \\
60  & 1.0453 & 1.0854 & 1.0705 & 1.0361 & 1.0428 & 1.0549 \\
70  & 1.0611 & 1.0432 & 1.0768 & 1.0438 & 1.0403 & 1.0673 \\
80  & 1.0350 & 1.0690 & 1.0690 & 1.0472 & 1.0315 & 1.0442 \\
90  & 1.0411 & 1.0570 & 1.0760 & 1.0462 & 1.0097 & 1.0573 \\
100 & 1.0483 &     & 1.0534 & 1.0377 & 1.0326 & 1.0454 \\
\bottomrule
\end{tabular}}
\end{table}
\newpage
\subsection{\rebut{Ablation studies}}\label{sec:ablation}
\begin{table}[H]
\centering
\caption{\rebut{Comparison of old-task retention after sequential fine-tuning on tasks A $\rightarrow$ B. Both methods start from the same stage-I checkpoint and train 50\% of attention heads per stage. In stage II, Partition trains the complementary heads, while Reuse continues training the stage-I heads. We report old-task accuracy (\%) on in-distribution and out-of-distribution evaluations, with changes from stage I. Higher stage-II accuracy is highlighted in bold. The performance before stage II training is highlighted in gray. RR denotes round-robin allocation.}}
\label{tab:partition-vs-nonpartition}
\small
\setlength{\tabcolsep}{4pt}
\rebut{
\resizebox{\linewidth}{!}{
\begin{tabular}{llccccccc}
\toprule
Model & Task sequence & Stage-I
& \multicolumn{2}{c}{After Stage I}
& \multicolumn{2}{c}{Stage II: Partition}
& \multicolumn{2}{c}{Stage II: Reuse} \\
\cmidrule(lr){4-5}\cmidrule(lr){6-7}\cmidrule(lr){8-9}
& & heads & In-dist & Out-dist & In-dist & Out-dist & In-dist & Out-dist \\
\midrule
Qwen3-1.7B & Math $\rightarrow$ Coding & Top 50\%
& \cellcolor[gray]{0.88}37.50 & \cellcolor[gray]{0.88}75.60
& \textbf{35.80} ($\downarrow 1.70$) & \textbf{75.20} ($\downarrow 0.40$)
& 28.90 ($\downarrow 8.60$) & 67.60 ($\downarrow 8.00$) \\
\midrule
Qwen3-4B & Math $\rightarrow$ IF & Top 50\%
& \cellcolor[gray]{0.88}59.30 & \cellcolor[gray]{0.88}88.40
& \textbf{60.00} ($\uparrow 0.70$) & \textbf{87.20} ($\downarrow 1.20$)
& 55.50 ($\downarrow 3.80$) & 84.40 ($\downarrow 4.00$) \\
\addlinespace[2pt]
& & RR 50\%
& \cellcolor[gray]{0.88}53.20 & \cellcolor[gray]{0.88}86.00
& 53.20 ($\leftrightarrow 0.00$) & \textbf{85.60} ($\downarrow 0.40$)
& \textbf{54.00} ($\uparrow 0.80$) & 83.40 ($\downarrow 2.60$) \\
\midrule
& Coding $\rightarrow$ IF & Top 50\%
& \cellcolor[gray]{0.88}69.00 & \cellcolor[gray]{0.88}51.40
& \textbf{68.30} ($\downarrow 0.70$) & \textbf{52.60} ($\uparrow 1.20$)
& 67.72 ($\downarrow 1.28$) & 46.80 ($\downarrow 4.60$) \\
\addlinespace[2pt]
& & RR 50\%
& \cellcolor[gray]{0.88}68.00 & \cellcolor[gray]{0.88}42.90
& 64.30 ($\downarrow 3.70$) & \textbf{34.60} ($\downarrow 8.30$)
& \textbf{67.99} ($\downarrow 0.01$) & 33.90 ($\downarrow 9.00$) \\
\bottomrule
\end{tabular}
}}
\end{table}

\begin{table}[H]
\centering
\caption{\rebut{Stability--plasticity comparison after sequential fine-tuning on tasks A $\rightarrow$ B. Configuration same as \Cref{tab:partition-vs-nonpartition}. We report out-of-distribution accuracy (\%) on both tasks, with old-task changes from stage I. Higher stage-II accuracy within each setting is highlighted in pink. The performance before stage II training is highlighted in gray.}}
\label{tab:partition-vs-reuse-full}
\small
\setlength{\tabcolsep}{4pt}
\rebut{
\resizebox{\linewidth}{!}{
\begin{tabular}{llllcc}
\toprule
Model & Task sequence & Stage-I & Train stage
& Old task A metric & New task B metric \\
\midrule
Qwen3-1.7B & Math $\rightarrow$ Coding & Top 50\% & I
& \cellcolor[gray]{0.88}75.60
& \cellcolor[gray]{0.88}37.00 \\
& & & II: Partition
& \cellcolor[RGB]{255,225,235}75.20 ($\downarrow 0.40$)
& 37.70 \\
& & & II: Reuse
& 67.60 ($\downarrow 8.00$)
& \cellcolor[RGB]{255,225,235}38.40 \\
\midrule
Qwen3-4B & Math $\rightarrow$ IF & Top 50\% & I
& \cellcolor[gray]{0.88}88.40
& \cellcolor[gray]{0.88}22.70 \\
& & & II: Partition
& \cellcolor[RGB]{255,225,235}87.20 ($\downarrow 1.20$)
& \cellcolor[RGB]{255,225,235}26.00 \\
& & & II: Reuse
& 84.40 ($\downarrow 4.00$)
& 23.67 \\
\addlinespace[3pt]
& & RR 50\% & I
& \cellcolor[gray]{0.88}86.00
& \cellcolor[gray]{0.88}23.30 \\
& & & II: Partition
& \cellcolor[RGB]{255,225,235}85.60 ($\downarrow 0.40$)
& 25.30 \\
& & & II: Reuse
& 83.40 ($\downarrow 2.60$)
& \cellcolor[RGB]{255,225,235}25.33 \\
\midrule
& Coding $\rightarrow$ IF & Top 50\% & I
& \cellcolor[gray]{0.88}51.40
& \cellcolor[gray]{0.88}24.30 \\
& & & II: Partition
& \cellcolor[RGB]{255,225,235}52.60 ($\uparrow 1.20$)
& 23.00 \\
& & & II: Reuse
& 46.80 ($\downarrow 4.60$)
& \cellcolor[RGB]{255,225,235}26.00 \\
\addlinespace[3pt]
& & RR 50\% & I
& \cellcolor[gray]{0.88}42.90
& \cellcolor[gray]{0.88}25.70 \\
& & & II: Partition
& \cellcolor[RGB]{255,225,235}34.60 ($\downarrow 8.30$)
& \cellcolor[RGB]{255,225,235}25.00 \\
& & & II: Reuse
& 33.90 ($\downarrow 9.00$)
& 24.67 \\
\bottomrule
\end{tabular}
}}
\end{table}
\begin{table}[H]
\centering
\caption{\rebut{Comparison of layer-level budget-fixed head selection on instruction following. We select either the top or bottom 50\% of attention heads within each layer, ensuring identical layer-wise training budgets. We report in-distribution and out-of-distribution accuracy (\%). Base model performance is highlighted in gray, and higher accuracy is highlighted in bold.}}
\label{tab:per-layer-if}
\small
\setlength{\tabcolsep}{6pt}
\rebut{
\begin{tabular}{lcccccc}
\toprule
Model & \multicolumn{2}{c}{Base}
& \multicolumn{2}{c}{Per-layer Top 50\%}
& \multicolumn{2}{c}{Per-layer Bottom 50\%} \\
\cmidrule(lr){2-3}\cmidrule(lr){4-5}\cmidrule(lr){6-7}
& In-dist & Out-dist & In-dist & Out-dist & In-dist & Out-dist \\
\midrule
Qwen3-1.7B
& \cellcolor[gray]{0.88}41.70 & \cellcolor[gray]{0.88}17.70
& \textbf{52.60} & \textbf{20.00}
& 48.80 & 19.67 \\
\midrule
Qwen3-4B
& \cellcolor[gray]{0.88}50.50 & \cellcolor[gray]{0.88}21.00
& \textbf{57.80} & \textbf{26.67}
& 55.50 & 24.33 \\
\midrule
Mistral-7B
& \cellcolor[gray]{0.88}27.40 & \cellcolor[gray]{0.88}18.30
& \textbf{84.60} & \textbf{33.33}
& 61.50 & 24.33 \\
\bottomrule
\end{tabular}}
\end{table}
\newpage
\subsection{Layer distributions of top-ranked heads}\label{sec:layer-dist}

\begin{figure}[H]
\centering
\includegraphics[width=0.8\textwidth]{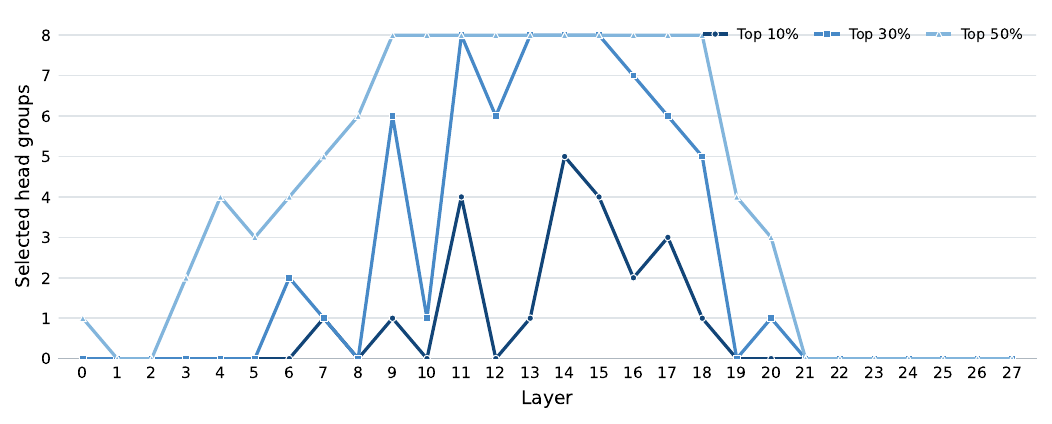}
\caption{Layer-wise distribution of selected attention-head groups for \lm{Qwen3-1.7B}, computed on the training set of Math.}
\label{fig:head-layer-qwen317b-math}
\end{figure}

\begin{figure}[H]
\centering
\includegraphics[width=0.8\textwidth]{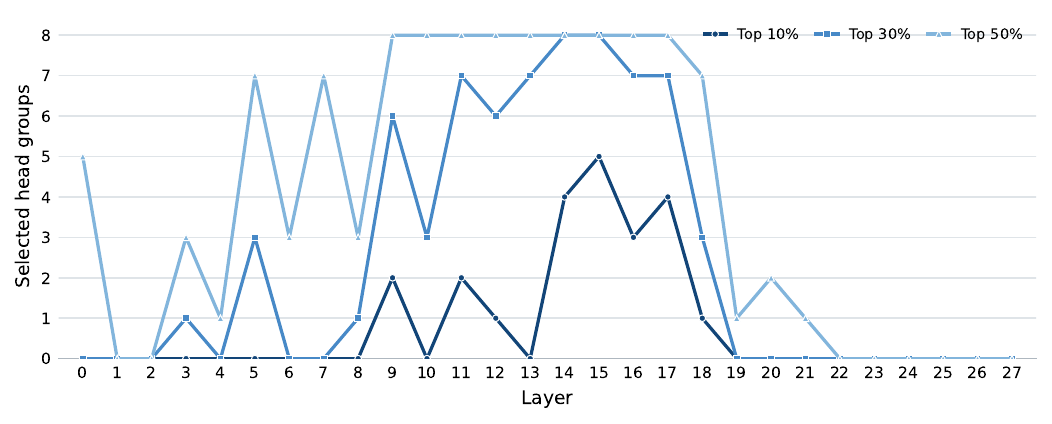}
\caption{Layer-wise distribution of selected attention-head groups for \lm{Qwen3-1.7B}, computed on the training set of Basic Coding.}
\label{fig:head-layer-qwen317b-mbpp}
\end{figure}

\begin{figure}[H]
\centering
\includegraphics[width=0.8\textwidth]{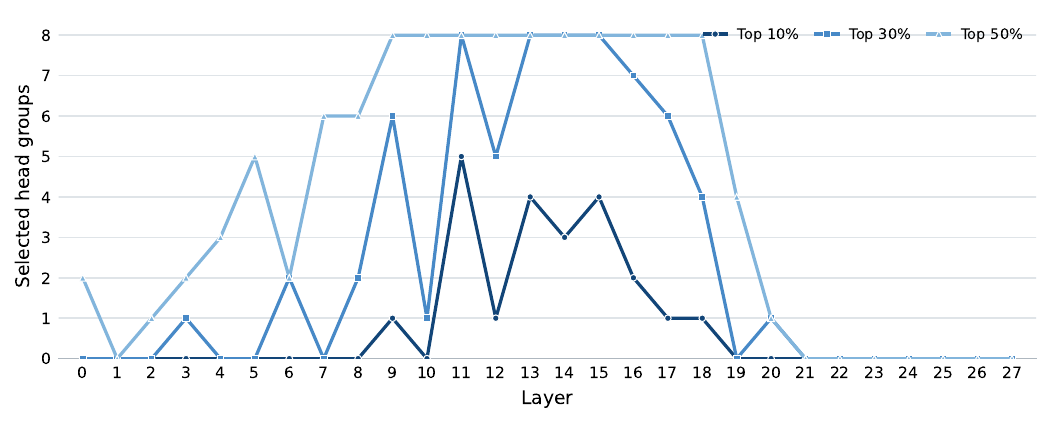}
\caption{Layer-wise distribution of selected attention-head groups for \lm{Qwen3-1.7B}, computed on the training set of Instruction-following.}
\label{fig:head-layer-qwen317b-ifeval}
\end{figure}

\begin{figure}[H]
\centering
\includegraphics[width=0.8\textwidth]{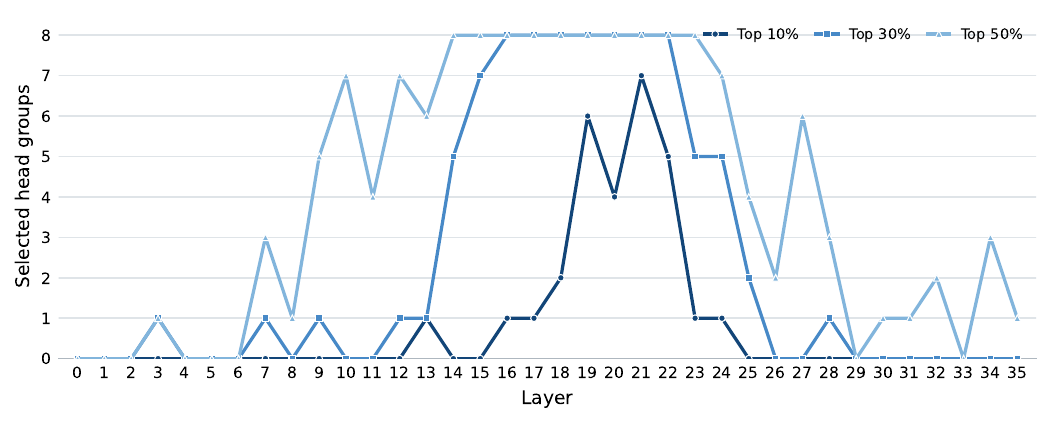}
\caption{Layer-wise distribution of selected attention-head groups for \lm{Qwen3-4B}, computed on the training set of Math.}
\label{fig:head-layer-qwen34b-math}
\end{figure}

\begin{figure}[H]
\centering
\includegraphics[width=0.8\textwidth]{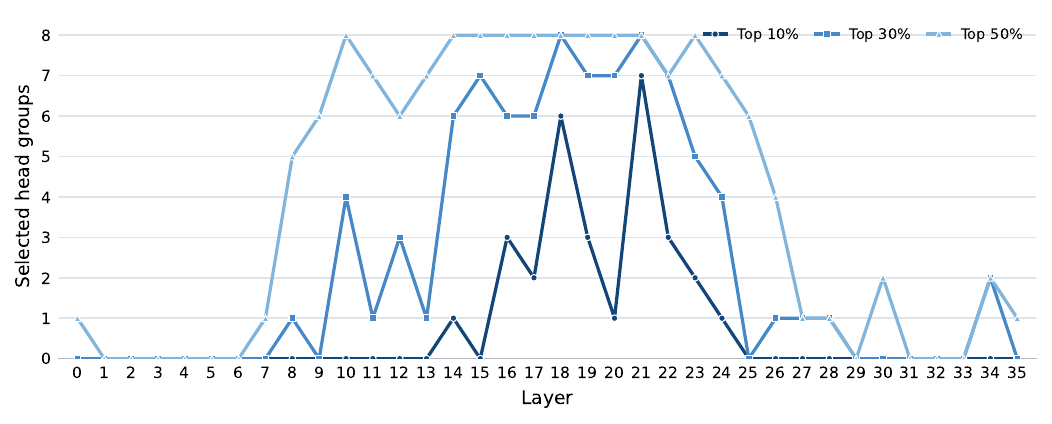}
\caption{Layer-wise distribution of selected attention-head groups for \lm{Qwen3-4B}, computed on the training set of Basic Coding.}
\label{fig:head-layer-qwen34b-mbpp}
\end{figure}

\begin{figure}[H]
\centering
\includegraphics[width=0.8\textwidth]{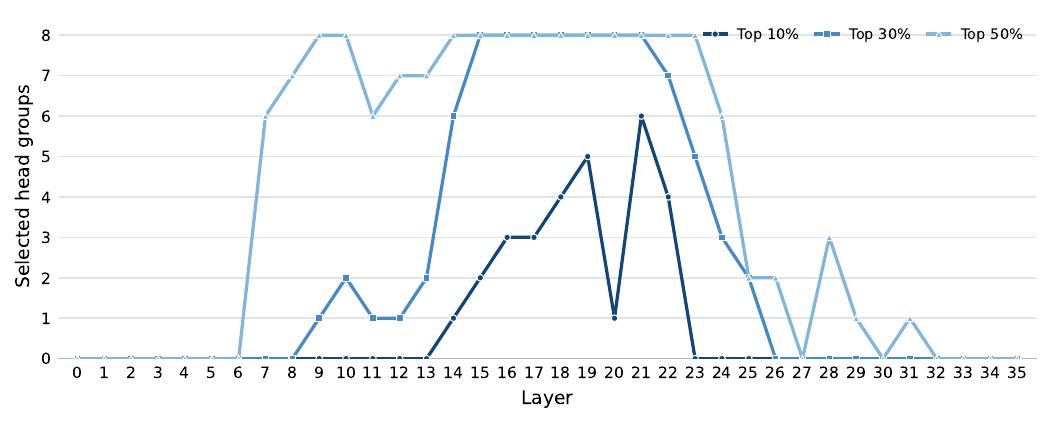}
\caption{Layer-wise distribution of selected attention-head groups for \lm{Qwen3-4B}, computed on the training set of Instruction-following.}
\label{fig:head-layer-qwen34b-ifeval}
\end{figure}

\begin{figure}[H]
\centering
\includegraphics[width=0.96\textwidth]{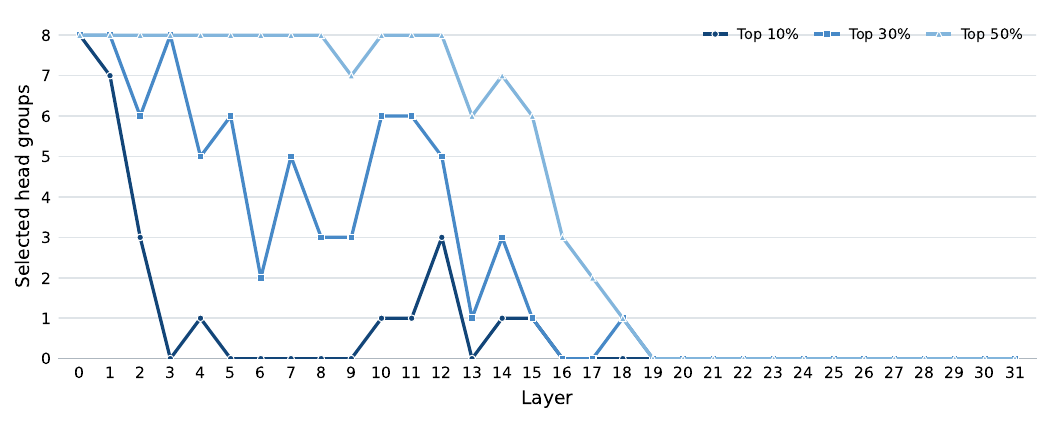}
\caption{Layer-wise distribution of selected attention-head groups for \lm{Mistral-7B-Instruction-v0.3}, computed on the training set of Instruction-following.}
\label{fig:head-layer-mistral7b}
\end{figure}
\newpage
\subsection{Hyperparameter tuning results}\label{sec:lr_tuning}
\begin{figure*}[h]
\centering
\includegraphics[width=0.78\textwidth]{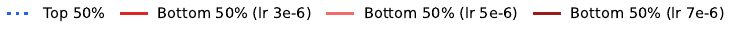}
\begin{minipage}[t]{0.48\textwidth}\centering\includegraphics[width=\linewidth]{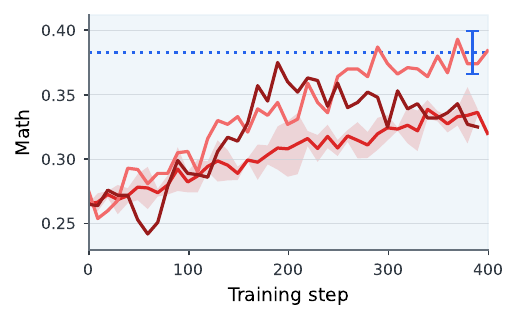}\par\small (a) \lm{Qwen3-1.7B}, Math\end{minipage}\hfill\begin{minipage}[t]{0.48\textwidth}\centering\includegraphics[width=\linewidth]{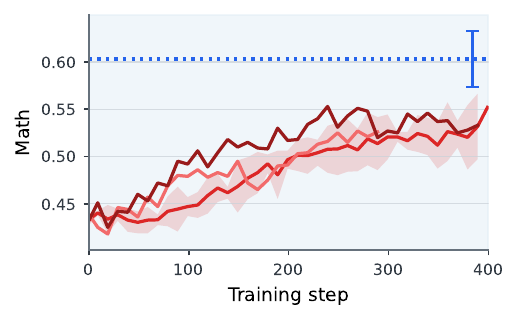}\par\small (b) \lm{Qwen3-4B}, Math\end{minipage}
\begin{minipage}[t]{0.48\textwidth}\centering\includegraphics[width=\linewidth]{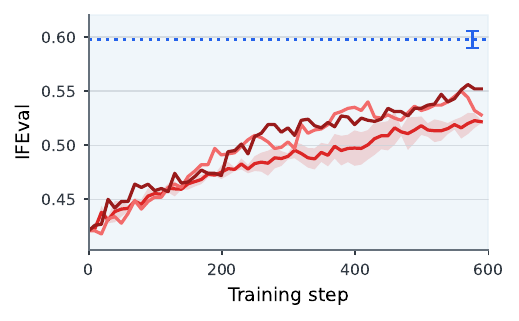}\par\small (c) \lm{Qwen3-1.7B}, IFEval\end{minipage}\hfill\begin{minipage}[t]{0.48\textwidth}\centering\includegraphics[width=\linewidth]{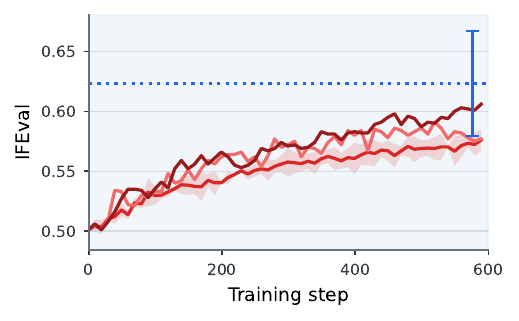}\par\small (d) \lm{Qwen3-4B}, IFEval\end{minipage}
\caption{Sparse RL training accuracy curves of \lm{Qwen3} series for Math and IF tasks. The learning rate for top-50\% attention-head tuning is fixed at $3\times10^{-6}$, while bottom-50\% attention-head tuning is evaluated with learning rates of $3\times10^{-6}$, $5\times10^{-6}$, and $7\times10^{-6}$. \rebut{Top-50\% attention-head tuning is shown as a horizontal line with an error bar indicating mean$\pm$std over three random rollout-seeds.}}
\label{fig:sparse-lr-sweep}
\end{figure*}

\newpage
\section{Implementation Details}\label{sec:implementation}
\paragraph{Codebase.} Our codebase is mainly based on miles~\citep{miles2026}~(\url{https://github.com/radixark/miles}). It is worth noting that our current sparse head tuning implementation is not yet fully optimized, though realizing its full efficiency benefits should be straightforward with modern training infrastructure. Our code is released at \url{https://github.com/srzer/head-sparse-rl}. The code was implemented with the assistance of Codex and Claude Code.

\paragraph{Datasets and benchmarks.} For math, we use a 1k subset of \texttt{DAPO-Math-17k}~(\url{https://huggingface.co/datasets/BytedTsinghua-SIA/DAPO-Math-17k}) for training, a disjoint 1k subset for in-distribution test, and \texttt{MATH-500}~(\url{https://huggingface.co/datasets/HuggingFaceH4/MATH-500}) for out-of-distribution test. For instruction-following, we similarly use a 1k subset of \texttt{RLVR-IFEval}~(\url{https://huggingface.co/datasets/allenai/RLVR-IFeval}) for training, a disjoint 1k subset for in-distribution test, and \texttt{IFBench}~(\url{https://huggingface.co/datasets/allenai/IFBench_test}) for out-of-distribution test. For basic coding, we use \texttt{MBPP}~(\url{https://huggingface.co/datasets/google-research-datasets/mbpp}) for training, \texttt{MBPP+}~(\url{https://huggingface.co/datasets/evalplus/mbppplus}) for in-distribution test, and \texttt{APPS-intro-test}~(\url{https://huggingface.co/datasets/codeparrot/apps}) for out-of-distribution test.
\rebut{Notably, \texttt{MBPP+} uses the same problem set as \texttt{MBPP} but with strengthened test cases, thus measuring performance under more rigorous evaluation.}

\paragraph{Models.} Limited by computation resources, our base models are \lm{Qwen3-1.7B}~\citep{qwen3technicalreport}~(\url{https://huggingface.co/Qwen/Qwen3-1.7B}), \lm{Qwen3-4B}~\citep{qwen3technicalreport}~(\url{https://huggingface.co/Qwen/Qwen3-4B}), and \lm{Mistral-7B-Instruct-v0.3}~(\url{https://huggingface.co/mistralai/Mistral-7B-Instruct-v0.3}).

\paragraph{Training hyper-parameters.}
Each rollout contains 16 prompts and 4 responses per prompt, with a global batch size of 64. We use temperature 1.0 and a maximum generation length of 3,072 tokens (2,048 for coding). Adam is used with learning rate $3\times10^{-6}$ ($1\times10^{-6}$ for full training, $2\times10^{-6}$ for \lm{Mistral-7B}), weight decay 0.1, and $(\beta_1,\beta_2)=(0.9,0.98)$. The policy-ratio clipping range is 0.2--0.28, with zero KL and entropy coefficients. Training runs for 400--600 rollout steps.

\paragraph{Evaluation hyper-parameters.}
We use greedy-decoding for evaluation. For each evaluation example, we generate a single response with a maximum length of 4,096 tokens, temperature 1.0, and top-$k$ equal to 1. We report pass@1 accuracy using the task-specific verifier for each task.

\paragraph{Computation resources.}
Experiments are run on NVIDIA A100 and RTX A6000 GPUs, using four GPUs per job (eight for \lm{Mistral-7B}), with a typical runtime of approximately one day.

\end{document}